\def\arxiv{1}

\if\arxiv1
\documentclass[pdflatex,sn-nature,iicol,dvipsnames]{sn-jnl}
\else
\documentclass[pdflatex,sn-nature,iicol,dvipsnames,lineno]{sn-jnl}
\fi

\usepackage{pifont}%
\usepackage{graphicx}%
\usepackage{multirow}%
\usepackage{amsmath,amssymb,amsfonts,amsthm}%
\usepackage{mathrsfs}%
\usepackage[title]{appendix}%
\usepackage{xcolor}%
\usepackage{listings}%
\usepackage{cleveref}%
\usepackage{subcaption}%
\usepackage{bbold}%
\usepackage{tikz}%
\usepackage{pgfplots}%
\usepackage{enumitem}%
\usepackage{lmodern}
\pgfplotsset{compat=1.18}%
\usetikzlibrary{pgfplots.colorbrewer}%
\usetikzlibrary{calc,positioning}%

\if\arxiv1
    \hypersetup{%
        colorlinks,
        breaklinks=true,
        plainpages=false,%
        citecolor=Fuchsia,
        linkcolor=PineGreen,
        urlcolor=MidnightBlue,
        bookmarksopen=true,%
        bookmarksnumbered=false,%
        bookmarksdepth=5%
    }
\fi

\if\arxiv1
    \newtheoremstyle{resultstyle}
    {0.5em}
    {0.5em}
    {\itshape}
    {}
    {\bfseries}
    {.}
    { }
    {\thmname{#1}\thmnumber{ #2}\thmnote{ (#3)}}%
    \newtheoremstyle{definitionstyle}
    {0.5em}
    {0.5em}
    {}
    {}
    {\bfseries}
    {.}
    { }
    {\thmname{#1}\thmnumber{ #2}\thmnote{ (#3)}}%
    \newtheoremstyle{remarkstyle}
    {0.5em}
    {0.5em}
    {\itshape}
    {}
    {\itshape}
    {.}
    { }
    {\thmname{#1}\thmnumber{ #2}\thmnote{ (#3)}}%
    \theoremstyle{resultstyle}
    \newtheorem{theorem}{Theorem}%
    \newtheorem{lemma}{Lemma}%
    \newtheorem{proposition}{Proposition}%
    \theoremstyle{remarkstyle}%
    \newtheorem{example}{Example}%
    \newtheorem{remark}{Remark}%
    \theoremstyle{definitionstyle}%
    \newtheorem{definition}{Definition}%
\else
    \theoremstyle{plain}
    \newtheorem{theorem}{Theorem}%
    \newtheorem{lemma}{Lemma}%
    \newtheorem{proposition}{Proposition}%
    \theoremstyle{remark}
    \newtheorem{remark}{Remark}%
    \theoremstyle{definition}%
    \newtheorem{definition}{Definition}%
\fi

\renewcommand{\a}{\alpha} 
\newcommand{\Bra}[1]{\left[ #1 \right]}
\newcommand{\Par}[1]{\left( #1 \right)}

\newcommand{\rp}{\oplus}
\newcommand{\rt}{\otimes}
\newcommand{\ro}{\mathbb{1}} 
\newcommand{\rz}{\mathbb{0}} 
\newcommand{\rconv}{\circledast}
\newcommand{\rint}{\oint}
\newcommand{\minus}{\scalebox{0.6}{$-$}}
\newcommand{\plus}{\scalebox{0.6}{$+$}}
\newcommand{\tropmin}{{T_{\minus}}}
\newcommand{\tropmax}{{T_{\plus}}}
\newcommand{\tropmaxmin}{{T_{\pm}}}
\DeclareMathOperator{\dom}{dom} 
\DeclareMathOperator{\sign}{sign} 

\newcommand{\SO}{\mathcal{S}} 
\newcommand{\TO}{\mathcal{T}} 
\newcommand{\RO}{\mathcal{R}} 
\newcommand{\PO}{\mathcal{P}} 
\newcommand{\bbR}{\mathbb{R}} 
\newcommand{\bbL}{\mathbb{L}} 
\newcommand{\bbC}{\mathbb{C}} 
\newcommand{\bbZ}{\mathbb{Z}} 
\newcommand{\bbN}{\mathbb{N}} 
\newcommand{\bbM}{\mathbb{M}} 
\newcommand{\cF}{\mathcal{F}} 
\newcommand{\cG}{\mathcal{G}} 
\newcommand{\cH}{\mathcal{H}} 
\newcommand{\cL}{\mathcal{L}} 
\newcommand{\cP}{\mathcal{P}} 
\newcommand{\pdv}[2]{\frac{\partial #1}{\partial #2}}

\begin{document}

\title[PDE-CNNs: Axiomatic Derivations and Applications]{PDE-CNNs: Axiomatic Derivations and Applications}

\author*[]{\fnm{Gijs} \sur{Bellaard}}\email{g.bellaard@tue.nl}

\author[]{\fnm{Sei} \sur{Sakata}}\email{seisakata@gmail.com}

\author[]{\fnm{Bart} \sur{M. N. Smets}}\email{b.m.n.smets@tue.nl}

\author[]{\fnm{Remco} \sur{Duits}}\email{r.duits@tue.nl}

\affil[]{\orgdiv{Department of Mathematics and Computer Science, CASA}, \orgname{Eindhoven University of Technology}, \orgaddress{\city{Eindhoven}, \country{The Netherlands}}}

\abstract{
PDE-based Group Convolutional Neural Networks (PDE-G-CNNs) use solvers of evolution PDEs as substitutes for the conventional components in G-CNNs. 
PDE-G-CNNs can offer several benefits simultaneously: fewer parameters, inherent equivariance, better accuracy, and data efficiency.

In this article we focus on Euclidean equivariant PDE-G-CNNs where the feature maps are two-dimensional throughout.
We call this variant of the framework a PDE-CNN.

From a machine learning perspective, we list several practically desirable axioms and derive from these which PDEs should be used in a PDE-CNN, this being our main contribution.
Our approach to geometric learning via PDEs is inspired by the axioms of scale-space theory, which we generalize by introducing semifield-valued signals.

Our theory reveals new PDEs that can be used in PDE-CNNs and we experimentally examine what impact these have on the accuracy of PDE-CNNs.
We also confirm for small networks that PDE-CNNs offer fewer parameters, increased accuracy, and better data efficiency when compared to CNNs.
}


\maketitle

\section{Introduction} \label{sec:introduction} 

Recently, PDE-based group equivariant convolution neural networks (PDE-G-CNNs) \cite{smets2023pde} were introduced. 
PDE-G-CNNs belong to the broad family of group equivariant convolution neural works (G-CNNs) \cite{cohen2016group}.
Unlike traditional CNNs, PDE based networks replace the usual components that make up a CNN layer, that being convolutions, max pooling, and non-linear activation functions, by solvers of evolution PDEs. 
The coefficients that govern the effect of the PDEs serve as the trainable parameters.
\Cref{fig:layer} contains a diagram of an example CNN layer and PDE layer, intended to illustrate the similarities and differences between them.

It is shown in \cite{pai2023functional,smets2023pde,bellaard2023analysis,bellaard2023geometric} that{\if\arxiv0\color{RoyalBlue}\fi, for vessel segmentation in medical images and digit classification problems,}
PDE-G-CNNs — in addition to being inherently equivariant — require fewer parameters, achieve higher accuracy, and are more data-efficient, in comparison to CNNs and G-CNNs.
{\if\arxiv0\color{RoyalBlue}\fi
From this perspective, PDE-G-CNNs can be preferable over other architectures in the aforementioned image processing tasks.
}

{\if\arxiv0\color{RoyalBlue}\fi
The PDE-G-CNN architecture is general in the sense that the feature maps \(f : M \to \bbR\) are defined on an arbitrary \textit{homogeneous space} \(M\) on which a \textit{Lie group} \(G\) acts.
However, the existing literature \cite{pai2023functional,smets2023pde, bellaard2023analysis,bellaard2023geometric} mainly concerns itself with \(M = \bbM_2 = \bbR^2 \times S^1\), the space of two-dimensional positions and orientations, together with \(G = \text{SE}(2) = \bbR^2 \rtimes SO(2)\), the group of two-dimensional rotations and translations.
In this article we will \textit{not} consider the general setting, or \(M = \bbM_2\) for that matter, and restrict ourselves to \(M = \bbR^2\) and \(G = \text{SE}(2)\), i.e. standard two-dimensional Euclidean space with its roto-translation symmetries, for simplicity.
We call this specific instance a PDE-CNN. 

In \cite{smets2023pde} the evolution PDEs that are used in the PDE-G-CNN architecture are
\begin{subequations} \label{eq:pdes_pde_g_cnn}
\begin{align}
    \text{convection} \quad \pdv{f}{t} &= v \cdot \nabla f\\
    \text{\(\alpha\)-diffusion} \quad \pdv{f}{t} &= - \tfrac{1}{\alpha} (-\Delta)^{\alpha/2} f, \alpha > 0 \label{eq:intro_frac_diff_pde}\\
    \text{\(\alpha\)-dilation} \quad \pdv{f}{t} &= + \tfrac{1}{\alpha} \| \nabla f \|^\alpha, \alpha > 1 \label{eq:intro_dilation_pde}\\
    \text{\(\alpha\)-erosion} \quad \pdv{f}{t} &= - \tfrac{1}{\alpha} \| \nabla f \|^\alpha, \alpha > 1 \label{eq:intro_erosion_pde}.
\end{align}
\end{subequations} 
Here \(f : M \times \bbR_{\geq 0} \to \bbR\) is some scalar field on \(M\) evolving over time \(t \geq 0\), with \(f(\cdot,0)\) set to an initial condition.
In the convection \(v : M \to TM\) denotes a vector field, and in the diffusion \(-(-\Delta)^{\alpha/2}\) denotes a (fractional) power of the Laplacian.
Intuitively, the PDEs respectively correspond to shifting, blurring, max pooling, and min pooling.

Importantly, the PDEs \eqref{eq:pdes_pde_g_cnn} are (implicitly) dependent on the Riemannian metric tensor field \(\cG\) that is chosen on the homogeneous space \(M\).
When using different Riemannian metrics, concepts such as Laplacian \(\Delta\), gradient \(\nabla\), and norm \(\| \cdot \|\) change accordingly, consequently altering the effect of the PDEs.
The parameters that determine the Riemannian metrics \(\cG\) are learned during the training of a PDE-based neural network.
The metric tensor field \(\cG\) is designed to be invariant to the Lie group \(G\), resulting in \(G\)-equivariant processing of the signals \(f : M \to \bbR\) \cite{smets2023pde}.

The diffusion, dilation, and erosion PDEs \eqref{eq:pdes_pde_g_cnn} used in \cite{smets2023pde} were not chosen arbitrarily; they satisfy properties considered desirable from a machine learning perspective.
For example, the PDEs are \textit{quasilinear} and \textit{equivariant} meaning that they 1) can be solved using convolutional-like operations allowing for fast parallel computation, and 2) allow for the design of inherently equivariant networks, resulting in an architecture that is robust and data-efficient \cite{mohamed2020data,cohen2019gauge}.
In fact, the desirable properties that we want the PDEs in PDE-based neural networks to have are essentially the axioms of \textit{scale-space theory} \cite{pauwels1995extended,iijima1959basic,alvarez1993axioms,florack2001nonlinear,heijmans2002algebraic,duits2004axioms,felsberg2004monogenic}. 
Later, in \Cref{sec:axioms} we explore and motivate these properties in more detail.

In this article we will be deriving in an axiomatic way which PDEs should be used in PDE-based neural networks.
This framework includes the PDEs that are currently already used \eqref{eq:pdes_pde_g_cnn}, but also reveals previously unused PDEs, meaning that the accuracy of PDE-G-CNNs could possibly be improved by adding them.
Our approach is inspired by the axioms of scale-space theory, which we generalize by introducing \textit{semifield-valued} signals, and motivated from a machine learning perspective.

}

\begin{figure*}
    \centering
    \begin{subfigure}[t]{0.42\linewidth}
        \centering
        \includegraphics[width=\linewidth]{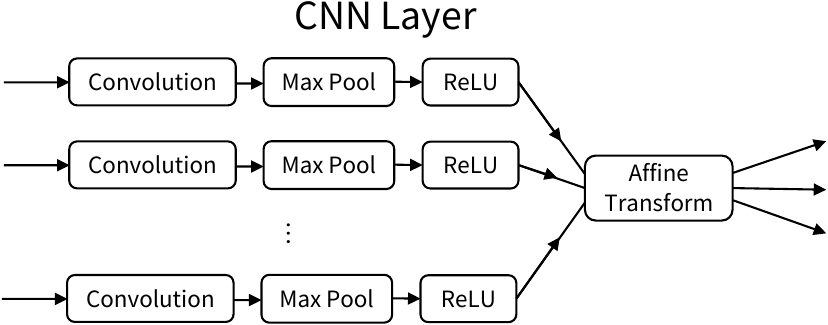}%
        \caption{CNN layer.}
        \label{fig:cnn_layer}
    \end{subfigure}
    \begin{subfigure}[t]{0.48\linewidth}
        \centering
        \includegraphics[width=\linewidth]{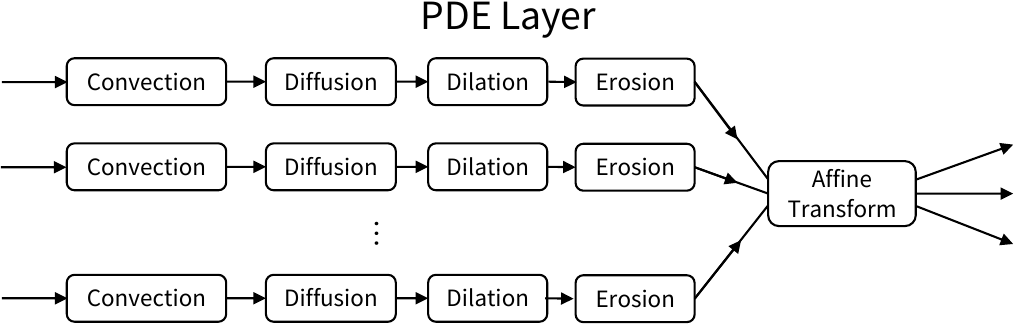}%
        \caption{PDE layer}
        \label{fig:pde_layer}
    \end{subfigure}
    \caption{Diagram of an example CNN layer and PDE layer. 
    {\if\arxiv0\color{RoyalBlue}\fi
    The vertical direction represents the channels.
    The arrows represent the ``flow'' of the feature maps through the parts that make up a layer.
    In machine learning terms, the affine transformation block is equivalent to a 2D convolution module with bias and 1x1 kernels.
    PDE based networks replace the usual components that make up a CNN layer, that being convolutions, max pooling, and non-linear activation functions, by solvers of evolution PDEs.
    The PDEs here are convection, diffusion, dilation, and erosion \eqref{eq:pdes_pde_g_cnn}.
    With ``solvers'' we mean the mapping from the initial condition \(f|_{t=0}\) to \(f|_{t=T}\).
    We can take \(T=1\) without loss of generality due to the scale-equivariance property of the PDEs (Axiom \ref{ax:r2_scaling}).}
    }
    \label{fig:layer}
\end{figure*}


\subsection{Contributions}

We list six axioms (\Cref{def:semifield_scale_space}) that a PDE used in a PDE-based neural network should satisfy.
The axioms are closely related to classical scale-space theory, but are more general in the sense that we permit semifield-valued signals.
The goal of this generalization is to allow for the discovery (or invention) of new PDEs that can be used in the design of PDE-based neural networks.

We will only consider semifields that are commutative and one-dimensional, and the domain of the PDEs will be the two-dimensional Euclidean space \(\bbR^2\).
To maintain a practical perspective, we will consistently connect the overarching theory using five example semifields: the linear, root, logarithmic, tropical min, and tropical max semifields.
    
From the axioms, we demonstrate in \Cref{res:explicit_form_reduced_kernel} that every semifield corresponds to a unique family of scale-spaces, this being the main theoretical contribution of the article.
This shows that PDE-based neural networks, in their current form, can be extended greatly by adding new PDEs that generate currently unused scale-spaces.

We experimentally assess how effective the incorporation of new semifields and their corresponding PDEs is in \Cref{sec:experiment_semifields}.

In \Cref{sec:data_efficiency} we verify that PDE-CNNs exhibit superior data efficiency, reduced parameter count, and competitive accuracy compared to traditional CNNs.

\subsection{Short Outline}

In \Cref{sec:background} we provide background on scale-spaces, semifields, and a non-exhaustive list of related literature.
In \Cref{sec:semifield_theory} we define semifields and all related structures and operations. 
In \Cref{sec:semifield_scale_space} we state the semifield scale-space axioms.
In \Cref{sec:consequences} we show that once a semifield is chosen a unique (one-parameter) family of scale-spaces arise (\Cref{res:explicit_form_reduced_kernel}).
In \Cref{sec:architecture} we briefly note on the architectural design of PDE-CNNs.
In \Cref{sec:experiments} we lay out two experiments and discuss their results.
In \Cref{sec:conclusion} we conclude the article.

{\if\arxiv0\color{RoyalBlue}\fi

\section{Background} \label{sec:background}

\subsection{Scale-Spaces}

The desired properties of PDEs in PDE-based neural networks are closely related to those of scale-space representations.
In fact, there is a one-to-one correspondence between scale-spaces and PDEs used in PDE-based neural networks.
In this section we will introduce and motivate the concept of scale-spaces, providing a few examples, and show to which PDE they correspond.
}

Real world scenes contain many different objects at different scales.
When a computer is tasked with analyzing an image of a scene there is no way for it to know beforehand at which scale(s) the interesting structures live.
One way to tackle this problem is to analyze the image of interest at \textit{all} scales.

In broad terms, a \textit{scale-space representation} of an image \(f_0\) is an ordered collection of images \(f_t\) where each successive image contains less and less detail; that is the smaller scales have been processed away.
The collection of images is usually indexed by the \textit{scale-parameter} \(t \geq 0\) with \(t=0\) being the original image.

Scale-spaces are a natural choice for computer vision solutions (either neural networks or classical methods) as they respect the inherent symmetries of images, that being translation, rotational, and scaling symmetries.
What we mean by this mathematically is that, for example, the scale-space \(g_t\) of a translated image \(g_0 = T_v f_0\), is equal to the translated scale-space of the original image: \(g_t = T_v f_t\).
Here \(T_v\) is the translation operator defined by \((T_v f)(x) = f(x - v)\).
Analogous statements hold for the rotation and scaling symmetries.
We say that creating the scale-space representation of an image is \textit{equivariant} with respect to translation, rotations and scalings. 


The prototypical, and most likely first \cite{iijima1959basic,weickert1999linear,koenderink1984structure}, example of a scale-space is the \textit{Gaussian scale-space} made by successive \textit{diffusing} (i.e.~\textit{blurring} or \textit{smoothing}) of the original image.
The Gaussian scale-space \(f_t\) of a two-dimensional image \(f_0 : \bbR^2 \to \bbR\) can be written as a \textit{linear convolution} \(*\) with a Gaussian \textit{kernel} \(k_t\):
\begin{equation}
\label{eq:intro_gaussian_scale_space_convolution}
\begin{split} 
    &\pdv{f}{t} = \frac{1}{2} \Delta f, \\
    &f_t = k_t * f_0, \quad
    k_t(x) = \frac{1}{2\pi t}\exp\Par{-\frac{\|x\|^2}{2t}},\\
    &(k_t * f_0)(x) = \int_{\bbR^2} k_t(x - y) f_0(y) {\rm d} y,
\end{split}
\end{equation}
where we used the notation \(f(x,t) = f_t(x)\).

Two other examples are the \textit{morphological scale-space} representations \cite{brockett1992evolution} made by successively \textit{dilating} or \textit{eroding} the original image.
The \(\alpha\)-dilation scale-space can be written as a non-linear \textit{dilating convolution} \(\boxplus\) with a kernel:
\begin{equation}
\label{eq:intro_dilation_scale_space_convolution}
\begin{split} 
    &\pdv{f}{t} = + \frac{1}{\a} \| \nabla f \|^\a, \\
    &f_t = k_t \boxplus f_0, \quad k_t(x) = -\frac{t}{\beta}\Par{\frac{\|x\|}{t}}^\beta, \\
    &(k_t \boxplus f_0)(x) = \sup_{y \in \bbR^2} k_t(x - y) + f_0(y), 
\end{split}
\end{equation}
where \(\beta\) is such that \(1/\a + 1/\beta = 1\).
The \(\alpha\)-erosion scale-space is created using an \textit{eroding convolution} \(\boxminus\):
\begin{equation}
\label{eq:intro_erosion_scale_space_convolution}
\begin{split} 
    &\pdv{f}{t} = - \frac{1}{\a} \| \nabla f \|^\a, \\
    &f_t = k_t \boxminus f_0, \quad k_t(x) = \frac{t}{\beta}\Par{\frac{\|x\|}{t}}^\beta, \\
    &(k_t \boxminus f_0)(x) = \inf_{y \in \bbR^2} k_t(x - y) + f_0(y). 
\end{split}
\end{equation}
The dilating and eroding convolutions are collected under the umbrella term \textit{morphological convolution}. 
This is because they are related by the identity \(-(-f\boxplus-g)=f\boxminus g\).

In \Cref{fig:scale_spaces_example_drive} the Gaussian, quadratic (\(\alpha=2\)) dilation, and quadratic erosion scale-spaces representations are visualized of a grayscale image of the fundus of the eye \cite{staal2004ridge}.

\begin{figure*}
    \centering
    \includegraphics[width=0.8\linewidth]{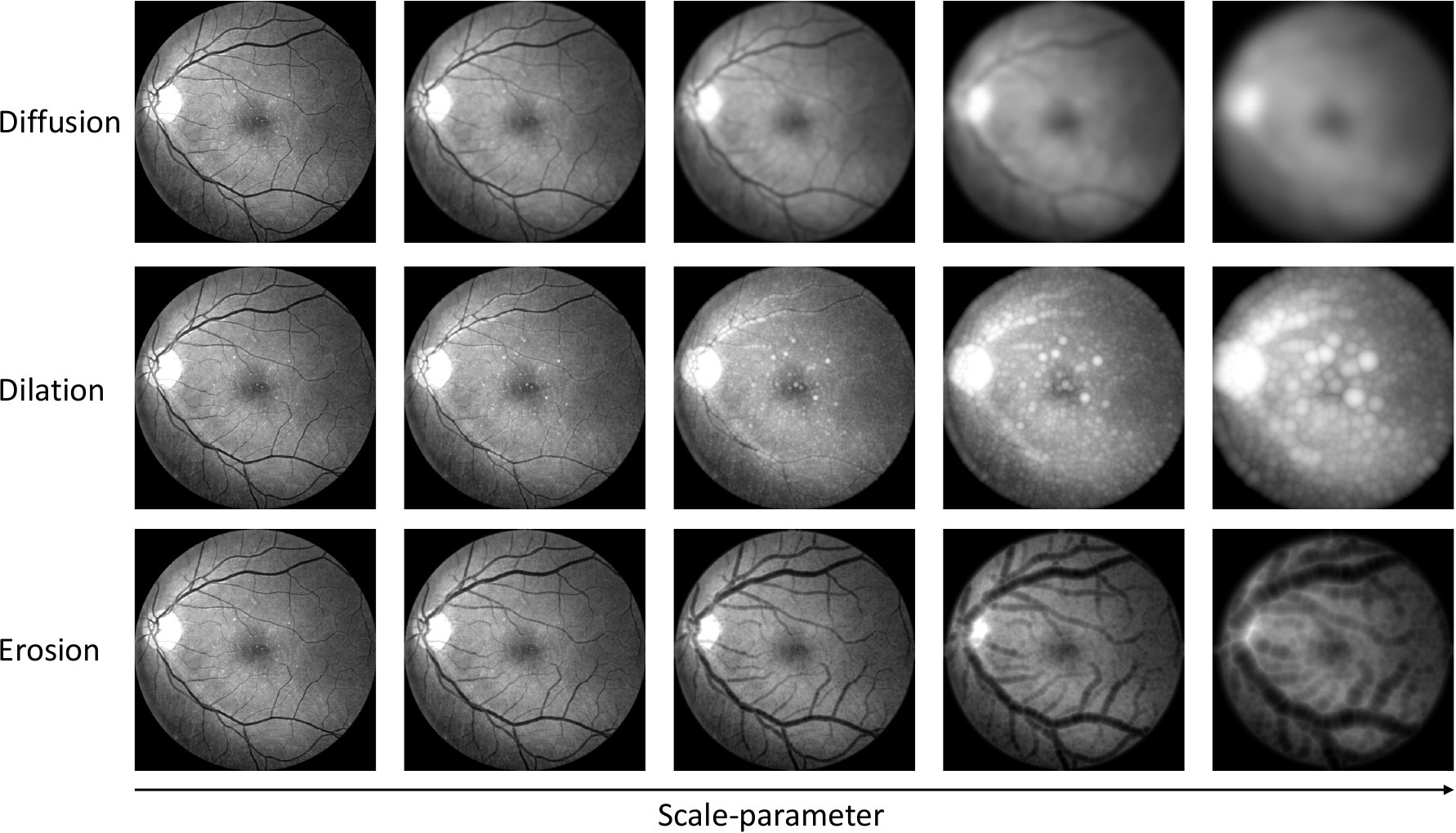}%
    \caption{The Gaussian \eqref{eq:intro_gaussian_scale_space_convolution}, quadratic (\(\alpha=2\)) dilation \eqref{eq:intro_dilation_scale_space_convolution}, and quadratic erosion \eqref{eq:intro_erosion_scale_space_convolution} scale-space representations of a grayscale image of the fundus of the eye at various scale-parameters. 
    In the Gaussian scale-space both white and black features fade away towards a uniform image.
    In the dilation scale-space the black details (low values), such as the vessels, vanish at bigger scales. 
    In the erosion scale-space the white details (high values), such as the space between vessels, are removed at higher scales. }
    \label{fig:scale_spaces_example_drive}
\end{figure*}



{\if\arxiv0\color{RoyalBlue}\fi

\subsection{Semifields \& Quasilinearity}

Every scale-space representation has a natural corresponding algebra called a \textit{semifield}.
In this section we will show which semifields correspond to the diffusion, dilation, and erosion scale-spaces.

}

The Gaussian scale-space representation \eqref{eq:intro_gaussian_scale_space_convolution} is linear in the sense that if one takes two images \(f_0, g_0 : \bbR^2 \to \bbR\) and two scalars \(a,b \in \bbR\), then the scale-space \(h_t\) of the image \(h_0 = a f_0 + b g_0\) is equal to \(h_t = a f_t + b g_t\).
But, in an analogous manner, the dilation scale-space \eqref{eq:intro_dilation_scale_space_convolution} is \textit{quasilinear} in the sense that the scale-space of the image \(h_0 = \max\{ a + f_0, b + g_0 \}\), where we interpret the maximum pointwise, is equal to \(h_t = \max\{a + f_t, b + g_t\}\).
In the same way, the erosion scale-space \eqref{eq:intro_erosion_scale_space_convolution} is quasilinear in the \(\min\) sense.
To define what we mean with quasilinear more precisely we need to introduce semifields.

A \textit{semifield} \((R,\rz,\ro,\rp,\rt)\) is an algebraic structure like a field but where we relax the requirement that the addition \(\rp\) has inverses.
The prototypical example of a semifield are the nonnegative real numbers \(L_{\geq 0} = (\bbR_{\geq 0},0,1,+,\times)\) with standard addition and multiplication.
We have already seen two other examples of semifields in the dilation and erosion scale-spaces. 
Namely, the so-called \textit{tropical max} semifield defined as \(\tropmax=(\bbR\cup\{-\infty\},-\infty,0,\max,+)\) and the \textit{tropical min} semifield \(\tropmin=(\bbR\cup\{\infty\},\infty,0,\min,+)\).
In the tropical semifields the minimum (or maximum) of two numbers becomes semifield addition, and normal addition becomes semifield multiplication.

With the definition of a semifield we can state the quasilinearity of a scale-space formally as \textit{semifield \(R\)-linearity}. 
So, like before, consider a semifield \(R\) and two semifield-valued images \(f_0, g_0 : \bbR^2 \to R\) and two elements \(a,b \in R\).
Then by a scale-space being \(R\)-linear we mean that the scale-space of the \(R\)-linear combination of images \(h_0 = (a \rt f_0) \rp (b \rt g_0)\) is equal to the \(R\)-linear combination of scale-spaces \(h_t = (a \rt f_t) \rp (b \rt g_t)\).
In other words, the operation that takes an image and returns its scale-space representation is a semifield linear operator.
For example, the Gaussian scale-space is \(L_{\geq 0}\)-linear, the quadratic dilation scale-space is \(\tropmax\)-linear, and the quadratic erosion scale-space is \(\tropmin\)-linear.

In \cite{pauwels1995extended,duits2004axioms} it is argued in an axiomatic way that the only linear scale-space representations correspond to solutions of the fractional diffusion (pseudo-)PDE system \eqref{eq:intro_frac_diff_pde}.
In a completely analogous manner, one can show \cite{martin2003families,schmidt2016morphological} that the only morphological scale-spaces, that being scale-spaces that are \(\tropmax\) or \(\tropmin\) linear, correspond to (viscosity) solutions of the \(\alpha\)-dilation \eqref{eq:intro_dilation_pde} and \(\alpha\)-erosion PDE \eqref{eq:intro_erosion_pde}.

{\if\arxiv0\color{RoyalBlue}\fi
These facts reveal something important: to discover new PDEs that can be used in PDE-based neural network we \textit{need} to generalize scale-space theory to semifields other than just \(L_{\geq 0}\), \(\tropmax\), and \(\tropmin\).
}

{\if\arxiv0\color{RoyalBlue}\fi

\subsection{Related Work} \label{sec:related_work}

    In this section we provide a nonexhaustive list of related scale-space literature.

    \textbf{Linear Scale-Spaces.}
    In \cite{iijima1959basic} the first \cite{weickert1999linear} axiomatic treatment of linear scale-space theory is presented.
    Axioms such as linearity, roto-translation equivariance, one-parameter semigroup property, and most notably, the scale equivariance, can all be found in Iijima's article, axioms we will also be using.
    Iijima shows that the Gaussian scale-space arises from his axioms, and that the Laplacian generates it.
    In \cite{pauwels1995extended,duits2004axioms} an extended class of linear scale-spaces is explored. 
    They derive that (fractional) powers of the Laplacian  \eqref{eq:intro_frac_diff_pde} are valid linear scale-space generators.
    Analysis of the scale-space axioms in the Fourier domain is extensively used, an approach we apply in the broader semifield setting.

    \textbf{Morphological Scale-Spaces.}
    In \cite{brockett1992evolution} it was shown for the first time \cite{heijmans2002algebraic} that morphological operators like dilations and erosions in image processing can be described in terms of PDEs.
    In \cite{dorst1995morphological} the slope transform is shown to be the morphological counterpart of the Fourier transform, and related to the Legendre-Fenchel transform. 
    The semifield Fourier transform we introduce reduces to the Legendre-Fenchel transform in the tropical semifield cases.
    In \cite{burgeth2005explanation,schmidt2016morphological} a connection between linear and morphological scale-spaces is described using the Cram\'{e}r transform.
    The Cram\'{e}r transform gives us a way to translate between the kernels of the linear and morphological scale-spaces.
    We will show that the kernels of all semifield scale-spaces have the same form in the Fourier domain (this being our main theorem), illuminating further the connection between the linear and morphological world.
    
    \textbf{Other Scale-Space Theory.}
    In \cite{alvarez1993axioms} an axiomatic approach to PDE-based scale-spaces is described. 
    The strength of this approach is that it also includes mean curvatures flows \cite{sapirobook} as highly powerful non-linear PDEs (also on Lie groups \cite{citti2016}). 
    Solutions of such non-linear PDEs may be solved with median filtering \cite{guichard1997}, however, they lack a semifield structure (taking mean/median are not associative binary operations), thus falling outside the scope of the theory presented here.
    In \cite{florack2001nonlinear} nonlinear scale-spaces are obtained by performing a monotonic transformation (known as a ``Cole-Hopf'' transform \cite[Ch.4.4]{evans2010partial}) on the grey-values of a standard linear scale-space and deducing what nonlinear PDE corresponds to the obtained evolution. 
    This transformation neatly bridges linear, logarithmic, and in the extreme cases, morphological scale-spaces, and we will also use this link.
    In \cite{heijmans2002algebraic} an algebraic framework for scale-spaces is given. 
    Importantly, their perspective is (initially) totally divorced from PDEs, convolutions, and kernels, and focuses solely on the evolution operator.
    We will define our semifield scale-spaces in the same manner.

    \textbf{Scale-Spaces in Machine Learning.}
    In \cite{jacobsen2016structured,pnitea2021resolution,tomen2021deep,saldanha2021frequency} the Gaussian scale-space and its spatial (fractional) derivatives are employed to design architectures that can learn filters at the appropriate scale by optimizing the scale parameter(s) during training.
    Architectures such as \cite{basting2023scale,luan2018gabor,romerobruintjes2021flexconv} also learn scale parameters.
    PDE-G-CNNs learn Riemannian metric tensor fields, which, due to the scale-equivariance of scale-space representations, is equivalent to learning scale-parameters.
    In this sense PDE-based neural networks are closely related to these ``scale learning'' architectures.
    In \cite{worall2019deep,sosnovik2020scale,lindeberg2022scale,sangalli2021scale} (discrete) Gaussian and morphological scale-space representations are used to create architectures that are scale equivariant.

}

\section{Semifield Theory} \label{sec:semifield_theory}

In this section we define semifields (\Cref{def:semifield}) and all mathematical structures and operations made from them.
This includes important concepts such as 
semimodules (\Cref{def:semimodule}), 
linearity (\Cref{def:semifield_linear}), 
measures (\Cref{def:semifield_measure}), 
integration (\Cref{def:semifield_integration}), 
convolution (\Cref{def:semifield_convolution}),
and Fourier transforms (\Cref{def:semifield_fourier_transform}).

\subsection{Semifield, Semimodules \& Linearity}

\begin{definition}[Semifield] \label{def:semifield}
    A (commutative) semifield \(R\) is a tuple $R=(R, \rz, \ro, \rp, \rt)$ where $\rp, \rt: R \times R \to R$ are two commutative and associative binary operations on \(R\) called semifield addition and multiplication, such that for all \(a,b,c \in R\):
    \begin{itemize}[label={}]
        \item \(a \rp \rz = a\),
        \item \(a \rt \ro = a\),   
        \item \(a \neq \rz : \exists a^{-1} : a \rt a^{-1} = \ro \),
        \item \(a \rt \rz = \rz\),
        \item \(a \rt (b \rp c)= (a \rt b) \rp (a \rt b)\).
    \end{itemize}
\end{definition}

In other words, a semifield is a field where we do not require to have ``negative elements'', that being additive inverses.

Throughout the article we will denote an arbitrary semifield-related operation with a circled version of the most closely related linear counterpart. Some example symbols are \(\rp\), \(\rt\), \(\rint\), and \(\rconv\), which respectively correspond to semifield addition, multiplication, integration, and convolution.

In this article we mainly consider the following semifields:
\begin{definition}[Semifields of Interest] 
    \label{def:semifields_interest}
    ~
    \begin{enumerate}[label={\alph*)}]
        \item The linear semifield \(L=(\bbR, 0, 1, +, \times)\) with the usual addition \(+\) and multiplication \(\times\). 
        We can restrict the set to \(\bbR_{\geq0}\) and we write \(L_{\geq 0}\) in that case. 

        \item The root semifields \(R_p = (\bbR_{\geq0}, 0, 1, \rp_p, \times)\) with $p\neq0$ where semifield addition is \(a \rp_p b := \sqrt[p]{a^p + b^p}\), and where semifield multiplication is normal multiplication.
        
        \item The logarithmic semifields \(L_\mu = (\bbR \cup \{\pm \infty\}, \pm \infty, 0, \rp_\mu, +)\) with $\mu\neq0$ where semifield addition is \(a \rp_\mu b := \tfrac{1}{\mu} \ln(e^{\mu a} + e^{\mu b})\), and where semifield multiplication is normal addition.
        If \(\mu>0\) we add \(-\infty\) to the ring to act as the additive identity, and if \(\mu<0\) we add \(+\infty\).

        \item The tropical
        max semifield \(\tropmax = (\bbR\cup\{-\infty\}, -\infty, 0, \max, +)\), where \(\max\) is semifield addition, and usual addition is semifield multiplication. 
        
        \item The tropical min semifield \(\tropmin = (\bbR\cup\{\infty\}, \infty, 0, \min, +)\), where \(\min\) is semifield addition, and usual addition is semifield multiplication. 
    \end{enumerate}
\end{definition}

The family of logarithmic semifields is interesting as in the limits one has:
\begin{equation}
\begin{split}
    \lim_{\mu \to +\infty} a \rp_\mu b &= \max(a, b), \\
    \lim_{\mu \to -\infty} a \rp_\mu b &= \min(a, b).
\end{split}
\end{equation}
Thereby, the family of logarithmic semifields \(L_\mu\) relate to the tropical semifields \(\tropmaxmin\) in the extreme cases of \(\mu\). 

\begin{definition}[Semifield Isomorphism] 
    \label{def:semifield_morphisms}
    Let \(R=(R, \rz, \ro, \rp, \rt)\) and \(\tilde R=(\tilde R, \tilde\rz, \tilde\ro, \tilde\rp, \tilde\rt)\) be two semifields. 
    A semifield isomorphism \(\varphi : R \to \tilde R\) is a bijective mapping that satisfies for all \(a,b \in R\):
    \begin{itemize}[label={}]
        \item \(\varphi(\rz) = \tilde\rz\),
        \item \(\varphi(\ro) = \tilde\ro\),
        \item \(\varphi(a \rp b) = \varphi(a) \mathbin{\tilde\rp} \varphi(b)\),
        \item \(\varphi(a \rt b) = \varphi(a) \mathbin{\tilde\rt} \varphi(b)\).
    \end{itemize}
    If there exists a semifield isomorphism between two semifields they are called isomorphic.
\end{definition}




\begin{proposition}[Some Semifields Isomorphism] 
    \label{res:semifield_isomorphism}
    ~
    \begin{itemize}
        \item The root semifields \(R_p\) are isomorphic to the nonnegative linear semifield \(L_{\geq 0}\), with the isomorphism \(\varphi_p : R_p \to L_{\geq 0}\) being \(\varphi_p(x) = x^p\).
        
        \item The logarithmic semifields \(L_\mu\) are isomorphic to the nonnegative linear semifield \(L_{\geq 0}\), with the isomorphism \(\varphi_\mu : L_\mu \to L_{\geq 0}\) being \(\varphi_\mu(x) = e^{\mu x}\).

        \item The tropical max semifield \(\tropmax\) is isomorphic to the tropical min semifield \(\tropmin\), with the isomorphism \(\varphi: \tropmax \to \tropmin\) being \(\varphi(x) = -x\).

       \item Informally, in the limit \(\mu \to \pm \infty\) the logarithmic semifields \(L_\mu\) ``converge'' to the tropical semifields \(\tropmaxmin\).
    \end{itemize}
\end{proposition}
The above proposition shows that although we defined five semifields of interest, as listed in \Cref{def:semifields_interest}, we are, in fact, only working with 2 non-isomorphic ones.


For the purpose of analysis we endow the semifields with a metric.
Normally, a linear structure \(X\) is endowed with a \textit{norm} \(\|\cdot\| : X \to \bbR_{\geq 0}\) and afterwards a metric \(d\) is defined through \(d(a,b) = \|a - b\|\).
This is not possible in our semifield setting as we do not necessarily have additive inverses (consider for example the tropical semifields).
\begin{definition}[Semifield Metric] \label{def:semifield_metric}
    Let \(R\) be a semifield. 
    A semifield metric \(\rho : R \times R \to \bbR_{\geq 0}\) is a metric such that for all \(a,b,c \in R\) we have:
    \begin{itemize}[label={}]
        \item \(\rho(c \rp a, c \rp b) \leq \rho(a,b)\),
        \item \(\rho(c \rt a, c \rt b) = \rho(c,\rz) \rho(a,b)\).
    \end{itemize}
\end{definition}
These properties are direct generalizations of the common notions of translation invariance and absolute homogeneity.
More importantly, they ensure that semifield addition \(\rp\) and multiplication \(\rt\) are continuous (w.r.t the metric).

\begin{definition}[Employed Semifield Metrics] 
    \label{def:employed_semifield_metrics}
    ~
    \begin{enumerate}[label={\alph*)}]
        \item In the linear semifield \(L\) case we use the metric \(\rho_L(a,b) = |a - b|\).
        \item In the root semifields \(R_p\) case we use the metric \(\rho_{R_p}(a,b) = |a^p - b^p|\).
        \item In the logarithmic semifields \(L_\mu\) case we use the metric \(\rho_{L_\mu}(a,b) = |e^{\mu a} - e^{\mu b}|\).
        \item In the tropical max semifield \(\tropmax\) case we use the metric \(\rho_{\tropmax}(a,b) = |e^a-e^b|\).
        \item In the tropical min semifield \(\tropmin\) case we use the metric \(\rho_{\tropmin}(a,b) = |e^{-a}-e^{-b}|\).
    \end{enumerate}
\end{definition}

The root and logarithmic semifield metrics are natural as they borrow the metric on the linear semifield \(L\) through the isomorphisms \(\varphi_p(x) = x^p\) and \(\varphi_\mu(x) = e^{\mu x}\), see \Cref{res:semifield_isomorphism}.
Similarly, the tropical min and max semifield metrics relate by their isomorphism \(\varphi(x) = -x\).

\begin{definition}[One-Dimensional Semifield] 
    \label{def:one_dimensional_semifield}
    Let \(R\) be a metric semifield. If \(R\) as a topological space (with the topology induced by the metric) is locally homeomorphic to one-dimensional Euclidean space we say it is one-dimensional.
\end{definition}



Just as mathematical rings and fields can be used to create modules and vector spaces, we define an analogous structure called a semimodule using semifields.

\begin{definition}[Semimodule] \label{def:semimodule}
    Let \(R=(R, \rz, \ro, \rp, \rt)\) be a semifield.
    An \(R\)-semimodule \(V=(V, \rp_V, \rt_V, \rz_V)\) over \(R\) is a set with a commutative and associative binary operation \(\rp_V : V \times V \to V\) called addition, and another binary operation \(\rt_V : R \times V \to V\) called (left) scalar multiplication, such that for all \(a,b \in R\) and \(u,v \in V\):
    \begin{itemize}[label={}]
        \item \(v \rp_V \rz_V = v\),
        \item \(\ro \rt_V v = v\),
        \item \(\rz \rt_V v = \rz_V\),
        \item \((a \rp b) \rt_V v = (a \rt_V v) \rp_V (b \rt_V v)  \),
        \item \(a \rt_V (u \rp_V v) = (a \rt_V u) \rp_V (a \rt_V v) \).
    \end{itemize}
\end{definition}
We do not write the subscript \(V\) on the operations of a semimodule \(V\) from here on out, as is usual.

Now that we have semimodules we can speak of semifield-linearity in its full generality.
The notion of semifield-linearity is totally analogous to the normal notion of linearity, therefore the name.

\begin{definition}[Semifield Linear] \label{def:semifield_linear}
    Let \(V_1, V_2\) be two semimodules over the same semifield \(R\). A mapping \(\varphi : V_1 \to V_2\) is called \(R\)-linear if for all $a,b \in R$ and $u,v \in V_1$ we have:
    \begin{equation}
        \varphi(a \rt u \rp b \rt v) = a \rt \varphi(u) \rp b \rt \varphi(v).
    \end{equation}
\end{definition}

\subsection{Functions, Measurability \& Integration}

The prototypical semimodule over a semifield is the space of all semifield-valued functions on a set.

\begin{definition}[Function Semimodule]
    Let \(R\) be a semifield.
    Consider the set \(F(\bbR^2,R)\) of all \(R\)-valued functions \(f :\bbR^2 \to R\).
    The set \(F(\bbR^2, R)\) forms an \(R\)-semimodule under point-wise semifield addition and multiplication.
    The semimodule \(F(\bbR^2, R)\) is called the function semimodule over \(\bbR^2\).
    More generally, any subsemimodule of \(F(\bbR^2, R)\) is also called a function semimodule over \(\bbR^2\).
    
\end{definition}

On the function semimodule \(F(\bbR^2, R)\) we define the following natural \(R\)-linear domain transformation operators:

\begin{definition}[Operators on Function Semimodule]
    ~
    \begin{itemize}
        \item \textbf{Translation Operator:}
        For all translation vectors \(v \in \bbR^2\) we define the translation operator $\TO_v$
        \begin{equation} \label{eq:translation_operator}
            (\TO_{v}f)(x) := f(-v+x).
        \end{equation}
            
        \item \textbf{Rotoreflection Operator:}
        For all orthonormal matrices \(Q \in \bbR^{2 \times 2}\) we define the rotoreflection operator $\RO_Q$
        \begin{equation} \label{eq:rotoreflection_operator}
            (\RO_{Q} f)(x) := f(Q^{-1}x).
        \end{equation}
        
        \item \textbf{Scaling Operator:}
        For all scalings \(s \in \bbR_{>0}\) we define the scaling operator $\SO_S$
        \begin{equation} \label{eq:scaling_operator}
            (\SO_{s} f)(x) := f \Par{ \frac{x}{s} }.
        \end{equation}

        \item \textbf{Pointwise Operator:}
        For all \(\varphi\) we define the pointwise operator \(\PO_\varphi\)
        \begin{equation} \label{eq:pointwise_operator}
            (\PO_\varphi(f))(x) := \varphi(f(x)).
        \end{equation}
    \end{itemize}
\end{definition}


To avoid pathological cases, we introduce standard measure theoretical concepts.

\begin{definition}[Measurable Space \& Set]
    Let $(X,d)$ be a complete metric space.
    We equip the space $X$ with the natural Borel sigma algebra $B$ induced by the metric $d$.
    This turns $X$ into a measurable space.
    A measurable set is any element of the Borel sigma algebra $B$.
\end{definition}

With the above definition we can turn both \(\bbR^2\) and any metric semifield \(R\) into a measurable space.

\begin{definition}[Measurable Function]
    Let $R$ be a metric semifield and $f : \bbR^2 \to R$ a function.
    The function $f$ is called a measurable function if the pre-image of any measurable set is a measurable set. 
\end{definition}

The set of measurable functions is broad enough to be well-behaved under pointwise limits, as the following lemma describes.

\begin{lemma} \label{res:pointwise_limit_measurable}
    Let $R$ be a metric semifield, and let $f(x) = lim_{n \to \infty} f_n(x)$ be the pointwise limit of measurable functions $f_n : \bbR^2 \to R$. 
    Then $f$ is also measurable.
\end{lemma}

A proof of a generalization of this lemma can be found at \cite{saz2014measurability}.
But there exists an even stronger statement that describes measurable functions as pointwise limits of indicator functions simple functions, which are made from indicator functions.

\begin{definition}[Indicator Function]
    Let \(R\) be a semifield and \(A \subseteq \bbR^2\) any set.
    We define the indicator function \(\ro_A\) of \(A\) as:
    \begin{equation}
        \ro_A(x) = \begin{cases}
            \ro & \text{ if } x \in A\\
            \rz & \text{ otherwise }
        \end{cases}
    \end{equation}
\end{definition}

\begin{definition}[Simple Function]
    Let \(R\) be a semifield.
    A simple function \(s : \bbR^2 \to R\) is a finite \(R\)-linear combination of indicator functions of measurable sets $A_i$.
    \begin{equation}
        s = \bigoplus_{i=1}^n a_i \rt \ro_{A_i},
    \end{equation}
    where each \(a_i \in R\).
\end{definition}

The link between measurable and simple functions is a follows.

\begin{lemma}
    Let \(R\) be a metric semifield and consider semifield-valued functions on \(\bbR^2\).
    The pointwise limit of a sequence of simple functions is measurable.
    Every measurable function is the pointwise limit of a sequence of simple functions.
\end{lemma}
\begin{proof}
    The pointwise limit of a sequence of simple functions being measurable follows immediately from \Cref{res:pointwise_limit_measurable}, as simple functions are measurable. 
    Showing that every measurable function is the pointwise limit of simple functions goes via a straightforward construction.
\end{proof}

For every semifield there is a natural associated class of functions.
We would like to specify this class in an axiomatic sense.
This is where the sum-approachable definition comes into play.
It is a restriction of the well-known statement that ``every measurable function is the limit of simple functions''.

\begin{definition}[Sum-Approachable]
    Let \(R\) be a metric semifield.
    A function \(f : \bbR^2 \to R\) is sum-approachable if there exists \(a_i \in R\) and \(A_i \subseteq \bbR^2\) open such that we have
    \begin{equation}
        f(x) = \lim_{n \to \infty} \bigoplus_{i=1}^n a_i \rt \ro_{A_i}(x).
    \end{equation}
    The semimodule of all sum-approachable functions $f : \bbR^2 \to R$ is denoted by $S(\bbR^2, R)$.
\end{definition}

There are two differences between sum-approachable and measurable: we only consider open sets, not measurable sets, and we have a limit of a semifield sum of indicator functions, not just a limit.

A function being sum-approachable is more restrictive than one might think at first sight.
The following lemma illustrates this by showing that in the tropical cases the sum-approachable functions enjoy the property of being semicontinuous, something that does \textit{not} happen in the linear case.

\begin{lemma}
    A sum-approachable function  \(f : \bbR^2 \to \tropmax\) is lower semicontinuous.
    A sum-approachable function  \(f : \bbR^2 \to \tropmin\) is upper semicontinuous.
\end{lemma}

\begin{proof}
    Consider the tropical max semifield case for the moment.
    Every indicator function \(\ro_A(x)\) with \(A \subseteq \bbR^2\) open is lower semicontinuous in this case.
    The limit-semifield-sum in the definition of sum-approachable turns into a pointwise supremum in this case.
    The pointwise supremum of lower semicontinuous functions is again lower semicontinuous\footnote{Let \(f(x) = \sup_n f_n(x)\). Let \(\varepsilon>0\) and \(x_0 \in \bbR^2\). Choose \(N\) such that \(f_N(x_0) > f(x_0) - \varepsilon/2\). Choose \(\delta > 0\) such that \(f_N(x) > f_N(x_0) - \varepsilon/2\) when \(|x - x_0| < \delta\). Then \(f(x) > f_N(x) > f_N(x_0) - \varepsilon/2 >  f(x_0) - \varepsilon\) \cite{matematleta2019semicontinuous}.}.
    Thus, every sum-approachable function \(f : \bbR^2 \to \tropmax\) is lower semicontinuous.
    \textit{Mutatis mutandis}, the exact same argument holds in the tropical min semifield case.
\end{proof}

\begin{definition}[Semifield Measure] 
    \label{def:semifield_measure}
    Let \(\Sigma\) be the Borel sigma algebra on \(\bbR^2\) and \(R=(R,\rz,\ro,\rp,\rt)\) a semifield.
    A semifield measure \(\mu : \Sigma \to R\) is a mapping that satisfies the following properties.
    \begin{itemize}
        \item \textbf{Nullity of Empty Set:} 
        \begin{equation}
            \mu(\varnothing)=\rz.
        \end{equation}
        \item \textbf{Disjoint Additivity:} 
        For all disjoint sets \(A, B \in \Sigma\):
        \begin{equation}
            \mu\Par{A \cup B}=\mu(A) \rp \mu(B),
        \end{equation}
        which we extend to countable collections of pairwise disjoint sets.
        \item \textbf{Unity of Unit Square:} 
        \begin{equation}
            \mu([0,1]^2) = \ro.
        \end{equation}
        \item \textbf{Translation Invariance:} 
        For all \(A \in \Sigma\) and \(v \in \bbR^2\):
        \begin{equation}
            \mu(A + v) = \mu(A).
        \end{equation}
        \item \textbf{Rotoreflection Invariance:} 
        For all \(A \in \Sigma\) and all orthonormal matrices \(Q \in \bbR^{2 \times 2}\):
        \begin{equation}
            \mu(QA) = \mu(A).
        \end{equation}
        \item \textbf{Scaling Equivariance:} 
        There exist a group homomorphism \(\chi : (\bbR_{>0},\times) \to (R\setminus\{\rz\},\rt)\) such that for all scalings \(s \in \bbR_{>0}\) and all \(A \in \Sigma\):
        \begin{equation}
            \mu(sA) = \chi(s) \rt \mu(A).
        \end{equation}
    \end{itemize}
\end{definition}

\begin{definition}[Employed Semifield Measure]
    \label{def:employed_measure}
    ~
    \begin{enumerate}[label={\alph*)}]
        \item In the linear semifield \(L\) case we use standard Lebesgue measure \(\lambda\). \(\mu_L(A) = \lambda(A)\).
        The scaling factor is \(\chi(s) = s^2\).
        
        \item In the root semifields \(R_p\) cases we use \(\mu_{R_p}(A) = \sqrt[p]{\lambda(A)}\).
        The scaling factor is \(\chi(s) = \sqrt[p]{s^2}\).
        
        \item In the logarithmic semifields \(L_\mu\) cases we use \(\mu_{L_\mu}(A) = \frac{1}{\mu} \ln \lambda(A)\).
        The scaling factor is \(\chi(s) = \frac{1}{\mu} \ln s^2\).
        
        \item In the tropical max semifield \(\tropmax\) case we use \(\mu_{\tropmax}(A) = 0\)\footnote{Remember that the semifield one \(\ro\) in the tropical max semifield \(\tropmax\) case is \(0\) (\Cref{def:semifields_interest}).}.
        The scaling factor is \(\chi(s) = 0\).
        
        \item In the tropical min semifield \(\tropmin\) case we use \(\mu_{\tropmin}(A) = 0\).
        The scaling factor is \(\chi(s) = 0\).
    \end{enumerate}
\end{definition}

\begin{definition}[Semifield Integration]   
    \label{def:semifield_integration}
    Let \(R\) be a metric semifield, \(S = S(\bbR^2,R)\) the space of sum-approachable functions, and \(\mu\) a semifield measure.
    Let \(\rint : (\dom(\rint) \subset S) \to R\) be a functional with the following properties.
    \begin{itemize}
        \item \textbf{Semifield Linearity:}
        For all \(a, b \in R\) and \(f, g \in \dom(\rint)\)
        \begin{equation} \label{eq:linearity_integration}
             \rint a \rt f \rp b \rt g = a \rt \left(\rint f \right) \rp b \rt \left(\rint g\right).
        \end{equation}

        \item \textbf{Indicator Function:}
        For all measurable sets \(A \subseteq \bbR^2\) we have
        \begin{equation} \label{eq:semifield_integration_indicator_function}
            \rint \ro_A = \mu(A).
        \end{equation}
        
        \item \textbf{Translation Invariance:}
        For all \(v \in \bbR^2\) and \(f \in \dom(\rint)\)
        \begin{equation} \label{eq:translation_integration}
            \rint \TO_v f =  \rint f.
        \end{equation}
        
        \item \textbf{Rotoreflection Invariance:}
        For all orthonormal matrices \(Q \in \bbR^{2 \times 2}\) and \(f \in \dom(\rint)\):
        \begin{equation} \label{eq:rotoreflection_integration}
            \rint \RO_Q f =  \rint f.
        \end{equation}

        \item \textbf{Scaling Equivariance:}
        For all scalings \(s \in \bbR_{>0}\) and \(f \in \dom(\rint)\):
        \begin{equation} \label{eq:scaling_integration}
            \rint \SO_s f = \chi(s) \rt \rint f,
        \end{equation}
        where \(\chi(s)\) is the scaling of the semifield measure \(\rint\) (\Cref{def:semifield_measure}).

        \item \textbf{Fubini:}
        For all \(f : \bbR^2 \times \bbR^2 \to R\) with both \(f(\cdot,y), f(x,\cdot) \in \dom(\oint)\), if one of the following integrals exists then they are equal:
        \begin{equation} \label{eq:fubini}
            \rint_y \rint_x f(x,y) = \rint_x \rint_y f(x,y).
        \end{equation}
    \end{itemize}
    We say that such a functional is a semifield integration.
    A function \(f\) that is in the domain of the semifield integration is called integrable.
\end{definition}

To emphasize over what slot we are integrating we may also write \(\rint_{x \in \bbR^2} f(x) = \rint f.\)
To emphasize over what semifield \(R\) the integration is taking place we may also write \(\rint = \rint^R\).

The first two properties of the semifield integration essentially nail down what the integration has to be.
That is, for every simple function \(s(x) = \bigoplus_{i=1}^n a_i \rt \ro_{A_i}(x)\) we have
\begin{equation}
    \oint s = \bigoplus_{i=1}^n a_i \rt \mu(A_i)
\end{equation}
by semifield linearity.
This then extends naturally to sum-approachable functions \(f(x) = \lim_{n \to \infty} \bigoplus_{i=1}^n a_i \rt \ro_{A_i}(x)\) by defining (with some caveats)
\begin{equation}
    \oint f = \lim_{n \to \infty} \bigoplus_{i=1}^n a_i \rt \mu(A_i).
\end{equation}
The caveats here being that we need requirements on the exact nature of the sequence of simple functions for the above to be well-defined.
To not get bogged down into the details we will just state what integration we will use for our relevant semifields, together with their domain of definition.
In the case of the tropical semifields we show in Appendix \ref{sec:tropical_integration_correct} that the upcoming semifield integration is indeed the correct one.

\begin{definition}[Employed Semifield Integration] 
    \label{def:employed_integration}
    ~
    \begin{enumerate}[label={\alph*)}]
        \item In the linear semifield \(L\) case we use standard Lebesgue integration
        \begin{equation}
            \rint^{L} f = \int_{x \in \bbR^2} f(x)\ {\rm d}x.
        \end{equation}
        The domain \(\dom(\rint^{L})\) is the space of Lebesgue integrable functions.
        
        \item In the root semifields \(R_p\) cases we use
        \begin{equation}
            \rint^{R_p} f 
            = \sqrt[p]{\int_{x \in \bbR^2} f(x)^p\ {\rm d}x}.
        \end{equation}
        The domain \(\dom(\rint^{R_p})\) consist of all functions \(f\) such that \(f^p\) is Lebesgue integrable.
        
        \item In the logarithmic semifields \(L_\mu\) cases we use
        \begin{equation}
            \rint^{L_\mu} f 
            = \frac{1}{\mu} \ln \int_{x \in \bbR^2} e^{\mu f(x)}\ {\rm d}x.
        \end{equation}
        The domain \(\dom(\rint^{L_\mu})\) consist of all functions \(f\) such that \(e^{\mu f}\) is Lebesgue integrable.
        
        \item In the tropical max semifield \(\tropmax\) case we use the supremum \(\sup\).
        \begin{equation}
            \rint^{\tropmax} f = \sup_{x \in \bbR^2} f(x) .
        \end{equation}
        The domain \(\dom(\rint^{\tropmax})\) consist of all functions \(f\) that are bounded from above.
        
        \item In the tropical min semifield \(\tropmin\) case we use the infimum \(\inf\).
        \begin{equation}
            \rint^{\tropmin} f = \inf_{x \in \bbR^2}  f(x).
        \end{equation}
        The domain \(\dom(\rint^{\tropmin})\) consists of all functions \(f\) that are bounded from below.
    \end{enumerate}
\end{definition}

The logarithmic and root semifield integration is natural as these semifields are isomorphic to the linear semifield, see \Cref{res:semifield_isomorphism}.
Additionally, the tropical max and min semifield integration are related through their isomorphism \(\varphi(x) = -x\). Indeed, one has \(\sup_{s \in S} s = - \inf_{s \in S}\{- s\}\).

\begin{definition}[Semifield Convolution] \label{def:semifield_convolution}
    Let \(R\) be a metric semifield.
    We define the semifield convolution \(\rconv\) of two integrable functions \(f, g \in \dom(\oint)\) as the new function \(f \rconv g \in \dom(\oint)\):
    \begin{equation}
        (f \rconv g)(x) := \rint_{y \in \bbR^2} f(x - y) \rt g(y).
    \end{equation}
\end{definition}

Showing that \(f \rconv g\) is indeed in \(\dom(\oint)\) is an immediate consequence of the Fubini property of semifield integration \eqref{eq:fubini}.
Moreover, the Fubini property gives us that the semifield convolution is associative:
\begin{equation}  
    \label{eq:semifield_convolution_associativity}
    f \rconv (g \rconv h) = (f \rconv g) \rconv h,
\end{equation}
and the translation invariance of semifield integration together with the commutativity of the semifield multiplication gives us that the semifield convolution is commutative.




We want to perform some analysis in our function spaces, so we need a (pseudo)metric \(\delta : S(\bbR^2,R) \times S(\bbR^2,R) \to \bbR_{\geq 0}\) (possibly with a restricted domain).
Similarly as before, when we introduced a metric on the semifields, we cannot make due with a norm on the function space as we have no additive inverses to turn the norm into a metric.

 \begin{definition}[Function Pseudometric] 
    \label{def:function_metric}
    Let \(R\) be a semifield with metric \(\rho\) and \(S = S(\bbR^2,R)\) the space of sum-approachable functions.
    A function (pseudo)metric \(\delta : S \times S \to \bbR_{\geq 0} \cup \{\infty\}\) is a (pseudo)metric such that for all \(f,g,h \in S\) and \(a \in R\) we have:
    \begin{itemize}[label={}]
        \item \(\delta(h \rp f, h \rp g)\leq\delta(f,g)\),
        \item \(\delta(a \rt f, a \rt g)=\rho(a,\rz)\delta(f,g)\).
    \end{itemize}
    We allow for the (pseudo)metric to return \(\infty\).
\end{definition}
Again, just as in \Cref{def:semifield_metric}, these properties are generalizations of the common notions of translation invariance and absolute homogeneity, and they ensure that both function addition \(\rp : S \times S \to S\) and function scalar multiplication \(\rt : R \times S \to S\) are continuous (in both slots).

\begin{definition}[Employed Function Pseudometric] 
    \label{def:employed_function_metric}
    ~
    \begin{enumerate}[label={\alph*)}]
        \item In the linear semifield \(L\) case we use
        \begin{equation}
            \delta_{L}(f,g) = \sqrt{\int_{\bbR^2} |f(x)-g(x)|^2 {\rm d}x}.
        \end{equation}
        \item In the root semifield \(R_p\) case we use
        \begin{equation}
             \delta_{R_p}(f,g) = \sqrt{\int_{\bbR^2} |f(x)^p-g(x)^p|^2 {\rm d}x}.
        \end{equation}
        \item In the logarithmic semifields \(L_\mu\) case we use
        \begin{equation}
             \delta_{L_\mu}(f,g) = \sqrt{\int_{\bbR^2} |e^{\mu f(x)}-e^{\mu g(x)}|^2 {\rm d}x}.
        \end{equation}
        \item In the tropical max semifield \(\tropmax\) case we use
        \begin{equation}
            \delta_{\tropmax}(f,g) = \sup_{x \in \bbR^2} |e^{f(x)}-e^{g(x)}|.
        \end{equation}
        \item In the tropical min semifield \(\tropmin\) case we use
        \begin{equation}
            \delta_{\tropmin}(f,g) = \sup_{x \in \bbR^2} |e^{-f(x)}-e^{-g(x)}|.
        \end{equation}
    \end{enumerate}
\end{definition}

Using the function (pseudo)metric we can make an appropriate function space:

\begin{definition}[Metric Function Space] 
    \label{def:normed_function_space}
    Let \(R\) be a metric semifield, \(S = S(\bbR^2,R)\) the space of sum-approachable functions, and \(\delta : S \times S \to \bbR_{\geq 0} \cup \{\infty\}\) a function (pseudo)metric.
    The function (pseudo)metric space \(H = H(\bbR^2, R, \delta)\) is defined as
    \begin{equation}
        H := \left\{ f \in S \mid \delta(\rz,f) < \infty \right\}.
    \end{equation}
    To turn it into an actual metric space we need to identify elements using the following natural equivalence relation \(\sim\).
    \begin{equation}
        f \sim g \iff \delta(f,g) = 0.
    \end{equation}
    This function space will be denoted with \(\cH = H/\sim\). 
\end{definition}


\begin{definition}[Employed Function Spaces]
    ~
    \begin{enumerate}[label={\alph*)}]
        \item In the linear semifield \(L\) case we have
        \begin{equation}
            \cH_L = \bbL^2(\bbR^2).
        \end{equation}
        \item In the root semifield \(R_p\) case we have
        \begin{equation}
            \cH_{R_p} = \{ f : \bbR^2 \to R_p \mid e^{\mu f} \in \bbL^2(\bbR^2) \}.
        \end{equation}
        \item In the logarithmic semifield \(L_\mu\) case we have
        \begin{equation}
            \cH_{L_\mu} = \{ f : \bbR^2 \to L_\mu \mid f^p \in \bbL^2(\bbR^2) \}.
        \end{equation}
        \item In the tropical max semifield \(\tropmax\) case we have
        \begin{equation}
            \cH_\tropmax = \{ f : \bbR^2 \to \tropmax \text{ l.s.c and b.f.a } \},
        \end{equation}
        where l.s.c means lower semicontinuous and b.f.a means bounded from above.
        \item In the tropical min semifield \(\tropmin\) case we have
        \begin{equation}
            \cH_\tropmin = \{ f : \bbR^2 \to \tropmin \text{ u.s.c. and b.f.b} \},
        \end{equation}
        where u.s.c means upper semicontinuous and b.f.b means bounded from below.
    \end{enumerate}
\end{definition}

\subsection{Fourier Transform}

We assume the existence of an injective Fourier transform that need only work on a very restricted class of semifield integrable functions.

\begin{definition}[Semifield Fourier Transform] \label{def:semifield_fourier_transform}
    Let \(R\) be a metric semifield.
    A semifield Fourier Transform $\cF_R : (\dom(\cF_R) \subseteq \dom(\oint)) \to \dom(\oint)$ is an operator satisfying (where we the drop the subscript \(R\) for conciseness):
    \begin{itemize}
        \item \textbf{Semifield Linearity:}
        For all \(a, b \in R\) and \(f, g \in \dom(\cF)\)
        \begin{equation}
            \cF \Par{a \rt f \rp b \rt g} = a \rt (\cF f) \rp b \rt (\cF g).
        \end{equation}
        
        \item \textbf{Convolution Property:}
        For all \(f, g \in \dom(\cF)\) with \(f \rconv g \in \dom(\cF)\)
        \begin{equation} 
            \label{eq:convolution_fourier}
            \cF(f \rconv g) = (\cF f) \otimes (\cF g).
        \end{equation} 

        \item \textbf{Rotoreflection Equivariance:}
        For all orthonormal matrices \(Q \in \bbR^{2 \times 2}\)
        \begin{equation} 
            \label{eq:rotoreflection_fourier}
            \cF \circ \RO_Q =  \RO_{Q^{-T}} \circ \cF.
        \end{equation}

        \item \textbf{Scaling Equivariance:}
        For all scalings \(s \in \bbR_{>0}\):
        \begin{equation} 
            \label{eq:scaling_fourier}
            \cF \circ \SO_s =  \chi(s) \rt \SO_{1/s} \circ \cF,
        \end{equation}
        where \(\chi(s)\) is the scaling of the semifield measure (\Cref{def:semifield_measure}).


        \item \textbf{Invertibility:}
        The domain \(\dom(\cF)\) is chosen such that the transform is injective and thus invertible on its image.
    \end{itemize}
\end{definition}

In the next definition we will specify the choice of semifield Fourier transform together with its appropriate choice of domain for all the semifields we consider (\Cref{def:semifields_interest}).
The choices we make here are sometimes more restrictive than strictly needed, but, as we will see in \Cref{sec:going_to_fourier_domain}, we only need to be able to take the semifield Fourier transform of a very ``small'' set of functions.

\begin{definition}[Employed Semifield Fourier Transform] 
    \label{def:employed_fourier}
    ~
    \begin{enumerate}[label={\alph*)}]
        \item In the linear semifield \(L\) case we use
        \begin{equation}
            (\cF_{L}f)(\omega) = \int_{\bbR^2} f(x) e^{-i \omega \cdot x} {\rm d} x.
        \end{equation}
        The domain \(\dom(\cF_{L})\) is chosen to be the space of even, continuous, and absolutely integrable functions, with absolutely integrable Fourier transforms.
        The inverse on its image is
        \begin{equation}
            (\cF_{L}^{-1} \hat f)(x) = \frac{1}{(2 \pi)^2} \int_{\bbR^2} \hat f(\omega) e^{i \omega \cdot x} {\rm d} \omega.
        \end{equation}
        
        \item In the root semifield \(R_p\) case we use
        \begin{equation}
            (\cF_{R_p}f)(\omega) = \sqrt[p]{\int_{ \bbR^2} f(x)^p e^{-i \omega \cdot x} {\rm d} x}.
        \end{equation}
        The domain  \(\dom(\cF_{R_p})\) is chosen such that \(f^p\) is in the domain of the linear Fourier transform \(\dom(\cF_{L})\), together with the restriction that the input of the \(p\)'th root is nonnegative.
        The inverse on its image is
        \begin{equation}
            (\cF^{-1}_{R_p} \hat{f})(x) = \sqrt[p]{ \frac{1}{(2 \pi)^{2}} \int_{\bbR^2} \hat{f}(\omega)^p  e^{i \omega \cdot x} {\rm d} \omega }.
        \end{equation}
        
        \item In the logarithmic semifield \(L_\mu\) case we use
        \begin{equation}
            (\cF_{L_\mu}f)(\omega) = \frac{1}{\mu} \ln \int_{x \in \bbR^2} e^{\mu f(x)} e^{-i \omega \cdot x} {\rm d} x.
        \end{equation}
        The domain  \(\dom(\cF_{L_\mu})\) is chosen such that \(e^{\mu f}\) is in the domain of the linear Fourier transform \(\dom(\cF_{L})\), together with the restriction that the input of the natural logarithm is positive.
        The inverse on its image is
        \begin{equation}
            (\cF^{-1}_{L_\mu}\hat{f})(x) = \frac{1}{\mu} \ln \Par{ \frac{1}{(2 \pi)^{2}} \int_{ \bbR^2} e^{\mu \hat{f}(\omega)} e^{i \omega \cdot x} {\rm d} \omega}.
        \end{equation}
        
        \item In the tropical max semifield \(\tropmax\) case we use
        \begin{equation} 
            (\cF_{\tropmax}f)(\omega) = \sup_{x \in \bbR^2} f(x) - \omega \cdot x.
        \end{equation}
        The domain \(\dom(\cF_{\tropmax})\) is chosen to be the space of even, continuous, concave, superlinear functions.
        The inverse on its image is
        \begin{equation}
            (\cF_{\tropmax}^{-1} \hat f)(x) = \inf_{\omega \in \bbR^2} \hat f(\omega) + \omega \cdot x.
        \end{equation}
        
        \item In the tropical min semifield \(\tropmin\) case we use
        \begin{equation}
            (\cF_{\tropmin}f)(\omega) = \inf_{x \in \bbR^2} f(x) - \omega \cdot x.
        \end{equation}
        The domain \(\dom(\cF_{\tropmin})\) is chosen to be the space of even, continuous, convex, superlinear functions.
        The inverse on its image is
        \begin{equation}
            (\cF_{\tropmin}^{-1} \hat f)(x) = \sup_{\omega \in \bbR^2} \hat f(\omega) + \omega \cdot x.
        \end{equation}
    \end{enumerate}
\end{definition}

A proof that these transforms satisfy the definition can be found in Appendix \ref{sec:employed_fourier_satisfy_definition}.

\begin{remark}
    \small
    The Laplace-like transform
    \begin{equation}
        (\cL f)(\omega) = \int_{\bbR^2} f(x) e^{-\omega \cdot x} {\rm d}x
    \end{equation}
    also satisfies \Cref{def:semifield_fourier_transform}.
    However, and this is also mentioned in \cite{schmidt2016morphological}, this transform is limited in its applicability because it is only finitely-valued for functions with super-exponential decay\footnote{This Laplace transform is two-sided thus resulting in this extreme condition. Also, we do not regard the transform as a conditionally convergent improper integral.}.
    Given this limitation of this transform, we instead use the normal Fourier transform.
\end{remark}

\begin{remark}
    \small 
    The above semifield Fourier transforms typically relate to transforms of the type
    \[
    (\mathcal{F} f)(\omega)=
    \rint f(x) \otimes \chi_{\omega}(x),
    \]
    where $\chi_{\omega}$ is an irreducible semifield-linear representation of $\bbR^2$, but we choose to express them in common Fourier/Fenchel transforms to keep a clear track of function space restrictions. 
\end{remark}

Even though we have used complex numbers in the Fourier transforms, the resulting transformed functions are always of the proper form \(\bbR^2 \to R\) due the domain consisting of even functions.
This means that we could have freely replaced the \(e^{-i \omega \cdot x}\) with \(\cos(\omega \cdot x)\). 
In other words, we could have instead used the \textit{Fourier cosine transform}.

\section{Semifield Scale-space} \label{sec:semifield_scale_space}

In this section we will state and motivate the semifield scale-space axioms, consider some examples semifield scale-spaces, and define what we mean with isomorphic scale-spaces. 

\subsection{Axioms} \label{sec:axioms}

In \cite{pauwels1995extended} it is stated that ``The only really nontrivial (and possibly too restrictive) assumption imposed on the scale-space operators, is that of linearity.''.
By generalizing to semifield linearity we sidestep this restrictive assumption, without making the theory too abstract to be practically useful.

Let us shortly motivate the semifield scale-space axioms from a machine-learning perspective (building upon similar findings in mathematical deep learning \cite{worall2019deep,sosnovik2020scale,sangalli2021scale,lindeberg2022scale}).
The semifield linearity (Axiom \ref{ax:r2_linearity}) together with the translation equivariance (Axiom \ref{ax:r2_equivariance}) will induce semifield convolutions that allow for fast parallel computations. 
The one-parameter semigroup property (Axiom \ref{ax:r2_one_parameter_semigroup}) together with the strong continuity (Axiom \ref{ax:r2_strong_continuity}) provides consistency and stability over evolution time.
The one-parameter semigroup property (Axiom \ref{ax:r2_one_parameter_semigroup}) together with the scaling equivariance (Axiom \ref{ax:r2_scaling}) allows us to constrain ourselves to a fixed end-time in a PDE sublayer, say \(t=1\), without loss of generality, further reducing the total parameter count.
The scaling, translation, and roto-reflection equivariance (Axioms \ref{ax:r2_scaling}, \ref{ax:r2_equivariance} and \ref{ax:r2_rotoreflection_equivariance}) of the scale-space allows for the design of inherently equivariant  networks, resulting in an architecture that is robust and data-efficient \cite{mohamed2020data,cohen2019gauge}.

\begin{definition}[Semifield Scale-space] \label{def:semifield_scale_space}
    Let \(R\) be a one-dimensional metric semifield (\Cref{def:one_dimensional_semifield}), \(\cH = \cH(\bbR^2,R,\delta)\) a corresponding metric function space (\Cref{def:normed_function_space}), and \(\Phi_t : \cH \to \cH\) be a family of operators, indexed by \(t \geq 0\). We call \(\Phi_t\) a semifield scale-space if it satisfies the following axioms:
    
    \begin{enumerate}
        \item \textbf{Semifield Linearity and Integral Operator:} \label{ax:r2_linearity}
        We require that \(\Phi\) is \(R\)-linear, that is for all \(f,g \in \cH\) and \(a,b \in R\):
        \begin{equation}
            \Phi_t(a \rt f \rp b \rt g) 
            = a \rt (\Phi_tf) \rp b \rt (\Phi_tg). 
        \end{equation}
        More specifically, for positive time \(t > 0\) we will assume that $\Phi_t$ can be written as an integral operator:
        \begin{equation}
            (\Phi_t f)(x)= \rint_{y \in \bbR^2} \kappa_t(x,y) \rt f(y),
        \end{equation}
        for some continuous kernel $\kappa_{t} : \bbR^2 \times \bbR^2 \to R$ with \(\kappa_t(x, \cdot)\) within the domain of the semifield Fourier transform \(\cF_R\) (\Cref{def:semifield_fourier_transform}).
    
        \item \textbf{One-Parameter Semigroup:} \label{ax:r2_one_parameter_semigroup}
        We require that \(\Phi_t\) forms a one-parameter semigroup
        , that is for all \(t,s \geq 0\) :
        \begin{equation}
            \Phi_t \circ \Phi_s = \Phi_{t+s} \text{ and } \Phi_0 = \text{id},
        \end{equation}
        where \(\text{id}\) is the identity map on \(\cH\).

        \item \textbf{Strong Continuity:} \label{ax:r2_strong_continuity}
        We require that \(\Phi_t f\) is continuous w.r.t. time \(t\) at any \(t_0>0\) for all \(f \in \cH\):
        \begin{equation}
            \lim_{t \to t_0} (\Phi_t f) = \Phi_{t_0} f,
        \end{equation}
        where the limit is taken in the metric function space \(\cH\)  (\Cref{def:normed_function_space}).
        
        \item \textbf{Scaling Equivariance:} \label{ax:r2_scaling}
        There exists a \textit{scaling power} \(\alpha > 0\) such that for all scalings \(s > 0\) and all times \(t \geq 0\):
        \begin{equation} 
            \Phi_{t} \circ \SO_{s} = \SO_{s} \circ \Phi_{t/s^\alpha}, 
        \end{equation}
        where $\SO_{s}$ is the scaling operator \eqref{eq:scaling_operator}.

        \item \textbf{Translation Equivariance:} \label{ax:r2_equivariance}
        We require that \(\Phi\) commutes with all translations \(v \in \bbR^2\):
        \begin{equation}
            \Phi_t \circ \TO_{v} = \TO_{v} \circ \Phi_t,
        \end{equation}
        where \(\TO_{v}\) is the translation operator \eqref{eq:translation_operator}.

        \item \textbf{Rotoreflection Equivariance:} \label{ax:r2_rotoreflection_equivariance}
        We require that \(\Phi\) commutes with all orthonormal matrices $Q \in \bbR^{2 \times 2}$:
        \begin{equation}
            \Phi_t \circ \RO_{Q} = \RO_{Q} \circ \Phi_t,
        \end{equation}
        where \(\RO_{Q}\) is the rotoreflection operator \eqref{eq:rotoreflection_operator}.

    \end{enumerate}
\end{definition}


{\if\arxiv0\color{RoyalBlue}\fi
Note that in our axioms we do not impose any restriction on the creation of new structures when transitioning from finer to coarser scales. 
This is in contrast to the requirement of \textit{causality} or \textit{non-enhancement of local extrema} in \cite{koenderink1984structure,lindeberg1997axiomatic}. 
}

In the linear semifield \(L\) case the linearity axiom in some sense already implies the integral operator axiom. 
The precise statement is known as the Schwartz kernel theorem, a main result in the theory of generalized functions/distributions. 
In the tropical semifield case a similar statement can be made, as demonstrated in \cite[Thm.2.1]{kolokoltsov1997idempotent}.
But for other semifields such a statement cannot be made just yet.
For simplicity, and to be on the safe side, we therefore assume the integral operator axiom.

The one-parameter semigroup property is a natural axiom in the sense that it implies that the (infinitesimal) evolution ``looks the same'' at all times \(t\). 
More precisely, the (strongly continuous) one-parameter semigroup property relates to the existence of a single generator that encapsulates the whole operator family. 
To understand, consider the linear semifield case, some initial \(f_0 \in \cH\), and its evolution \(f_t := \Phi_t(f_0)\). 
From the one-parameter semigroup axiom we have:
\begin{equation}
\begin{split} 
    f_{t+h} - f_t
    &= \Phi_{h+t}f_0 - \Phi_t f_0\\
    &= \Phi_h \Phi_t f_0 - \Phi_t f_0\\
    &= \left(\Phi_h - \Phi_0 \right)f_t.
\end{split}
\end{equation}
dividing by \(h\) and taking the limit \(h \downarrow 0\) 
in conjunction with strong continuity, 
we get the time-invariant evolution equation:
\begin{equation} 
    \label{eq:no_details_generator}
    \frac{df_t}{dt} = \Psi f_t, \text{ where } \Psi := \lim_{h \downarrow 0} \frac{\Phi_h - \Phi_0}{h}.
\end{equation}
The operator \(\Psi : D(\Psi) \to \cH\) is called the generator of the operator family \(\Phi_t\) and its natural domain $D(\Psi)$ consists of all functions $f \in \cH$ for which the above limit makes sense.
Typically, this domain will be dense in $\cH$. 

Thus, we can interpret \(\Phi_t\) as the solution operator of an evolution equation.
Given that the generator exists, it is possible through various means, for example the spectral theorem \cite{Rudinbook}, to give meaning to the expression:
\begin{equation}
    \Phi_t = e^{t\Psi},
\end{equation}
which can be used to quickly confirm (at least formally) that:
\begin{equation}
   \frac{d\Phi_t}{dt} = \frac{d}{dt}(e^{t\Psi}) = \Psi e^{t\Psi} = \Psi \Phi_t,
\end{equation}
which corresponds what we already saw in \eqref{eq:no_details_generator}.

\begin{remark}
    \small Given the existence of a generator \(\Psi\), the scaling equivariance axiom can be equivalently written as \(\Psi \circ \SO_{s} = \frac{1}{s^\alpha} \SO_{s} \circ \Psi\), revealing that the scaling equivariance can also be understood as a sort of \(\alpha\)-homogeneity of the generator.
\end{remark}

An easy and illustrative example of a generator together with its operator family is the derivative operator and the family of translation operators in one-dimensional space:
\begin{equation}
    (\Phi_t f)(x) = f(x + t), \quad \Psi = \frac{d}{dx}.
\end{equation}

A well-known related theorem in functional analysis is Stone's Theorem.
This theorem shows that there is a one-to-one correspondence between strongly continuous unitary one-parameter semigroups and (possible unbounded) densely defined self-adjoint operators on a Hilbert space.

In our case getting everything precise is made difficult by the fact that we want to generalize to semifields other than the linear semifield.
For example, we cannot even directly make sense of \eqref{eq:no_details_generator} for general semifields as there is not necessarily a \(-\) operation: we only have \(\rp\).
Given these obstacles, we will not attempt to rigorously prove that every semifield scale-spaces corresponds to a PDE, but will state the related PDEs in our primary cases of interest (\Cref{def:scale_spaces_interest}).

The scaling equivariance says that the scale-space representation of a scaled image should be a scaled version of the scale-space representation of the original image. 
In a sense we want a scale-space that does not ``care'' about absolute scale: it should qualitatively looks the same no matter the starting scale of the input. 
The translation and rotoreflection equivariance requirements are also not surprising: the Euclidean plane has its natural translation and rotoreflection symmetries, and demanding the scale-space to respect these is commonplace.

{\if\arxiv0\color{RoyalBlue}\fi
In PDE-based neural networks the trainable parameters take the form of Riemannian metrics \(\cG\) on a homogeneous space \(M\).
In the case of PDE-CNNs, that being \(M=\bbR^2\), this reduces to an inner product \(\cG : \bbR^2 \times \bbR^2 \to \bbR\) which we can always write as \(\cG(x,y) = x^\top G y\), where \(G\) is the corresponding \textit{Gram matrix}.
In the axioms we implicitly make, without loss of generality, the assumption that we use the ``standard'' inner product on \(\bbR^2\), namely \(\cG(x,y)=x^\top y\).
Later in \Cref{sec:architecture} we explain how we bring back general inner products in the PDE-CNN architecture.
}

\subsection{Examples}

\begin{definition}[Scale-spaces of Interest] 
    \label{def:scale_spaces_interest}
    ~
    \begin{enumerate}[label={\alph*)}]
        \item The Gaussian scale-space over the linear semifield \(L\):
        \begin{align*}
            (\Phi_t f)(x) &= \int_{y \in \bbR^2} \kappa_t(x,y) \times f(y)\ { \rm d}y, \\\kappa_t(x,y) &= \frac{1}{2\pi t} \exp \Par{-\frac{1}{2}\frac{\|x-y\|^2}{t}},
        \end{align*}
        which correspond to solutions of
        \begin{equation}
            \pdv{f}{t} = \frac{1}{2} \Delta f.
        \end{equation}
        The scaling power is \(\alpha=2\).

        \item The (quadratic) root scale-spaces over the root semifields \(R_p\):
        \begin{align*}
            (\Phi_t f)(x) &= \rint_{y \in \bbR^2}^{R_p} \kappa_t(x,y) \times f(y), \\
            \kappa_t(x,y) &= \frac{1}{\sqrt[p]{2\pi t}} \exp \Par{-\frac{1}{2p}\frac{\|x-y\|^2}{t}},
        \end{align*}
        which correspond to solutions of
        \begin{equation} \label{eq:quadratic_root_scale_space_pde}
            \pdv{f}{t} = \frac{p-1}{f}\ \frac{1}{2} \|\nabla f \|^2 + \frac{1}{2} \Delta f.
        \end{equation}
        The scaling power is \(\alpha=2\).
        
        \item The (quadratic) logarithmic scale-spaces over the logarithmic semifields \(L_\mu\):
        \begin{align*}
            (\Phi_t f)(x) &= \rint_{y \in \bbR^2}^{L_\mu} \kappa_t(x,y) + f(y), \\
            \kappa_t(x,y) &= - \frac{1}{\mu} \ln(2\pi t) -\frac{1}{2\mu}\frac{\|x-y\|^2}{t},
        \end{align*}
        which correspond to solutions of
        \begin{equation} \label{eq:quadratic_log_scale_space_pde}
            \pdv{f}{t} = \mu\frac{1}{2} \|\nabla f \|^2 + \frac{1}{2} \Delta f.
        \end{equation}
        The scaling power is \(\alpha=2\).

        \item The \(\alpha\)-dilation scale-space over the tropical max semifield \(\tropmax\):
        \begin{align*}
            (\Phi_t f)(x) &= \sup_{y \in \bbR^2} \kappa_t(x,y) + f(y), \\
            \kappa_t(x,y) &= -\frac{t}{\beta}\Par{\frac{\|x-y\|}{t}}^\beta,
        \end{align*}
        with \(1/\alpha + 1/\beta = 1\), which correspond to (viscosity) solutions of
        \begin{equation}
            \pdv{f}{t} = \frac{1}{\alpha} \|\nabla f\|^\alpha.
        \end{equation}
        The scaling power is \(\alpha\).
        
        \item The \(\alpha\)-erosion scale-space over the tropical min semifield \(\tropmin\):
        \begin{align*}
            (\Phi_t f)(x) &= \inf_{y \in \bbR^2} \kappa_t(x,y) + f(y), \\
            \kappa_t(x,y) &= \frac{t}{\beta}\Par{\frac{\|x-y\|}{t}}^\beta,
        \end{align*}
        with \(1/\alpha + 1/\beta = 1\), which correspond to (viscosity) solutions of
        \begin{equation}
            \pdv{f}{t} = - \frac{1}{\alpha}\|\nabla f\|^\alpha.
        \end{equation}
        The scaling power is \(\alpha\).

    \end{enumerate}
\end{definition}

 
The operators \(\Phi_t\) above solve the corresponding PDEs and one readily checks that the kernels \(\kappa_t(\cdot, y)\) satisfy the PDE for all \(y \in \bbR^2\). 
For example, consider the quadratic (\(\alpha=2\)) dilation scale-space and \(k_t(x) := \kappa_t(x,0) = -\tfrac{1}{2}\tfrac{\|x\|^2}{t}\):
\begin{equation}
    \pdv{k_t}{t} = \frac{1}{2}\frac{\|x\|}{t^2}, \quad \| \nabla k_t\|^2 = \frac{\|x\|^2}{t^2}.
\end{equation}
So, indeed, \(k_t\) satisfies the dilation PDE.
The same check can be done for the other scale-spaces.

\begin{remark}
    \small The Schr\"{o}dinger equation also generates a scale-space representation in the space \(\bbL^2(\bbR^2;\bbC)\), in the sense that it satisfies the linearity axiom and axioms 2-6.
    However, it does not fit in the theory here as the complex numbers do not form a one-dimensional semifield, and the corresponding kernel \(\kappa_t(x,\cdot)\) is not square integrable. 
    In T. Kraakman's master's thesis \cite{kraakman2023construction} PDE-based neural networks using the Schr\"{o}dinger equation are investigated and implemented. 
    They require more memory than our classical PDE-Based CNNs for only a small accuracy gain in practice so far. 
\end{remark}

\subsection{Isomorphic Scale-spaces}

In \Cref{res:semifield_isomorphism} we saw that the nonnegative linear, root, and logarithmic semifields are isomorphic, with the same being true for the tropical ones.
It seems natural then that there also exist isomorphisms between the corresponding semifield scale-spaces.
Let us start by clarifying what we mean by two semifield scale-spaces being isomorphic.

\begin{definition}[Semifield Scale-space Isomorphism]
    \label{def:scale_space_isomorphism}
    Let \(R\) and \(\tilde R\) be two semifields, and 
    let \(\Phi_t\) and \(\tilde\Phi_t\) be two semifield scale-spaces over \(R\) and \(\tilde R\) respectively.
    We say the two scale-spaces are isomorphic if there exists a semifield isomorphism  \(\varphi : \tilde R \to R\) such that
    \begin{equation}
        \Phi_t \circ \PO_\varphi = \PO_{\varphi} \circ \tilde\Phi_t,
    \end{equation}
    where \(\PO_\varphi : F(\bbR^2, \tilde R) \to F(\bbR^2, R)\) is the pointwise operator \eqref{eq:pointwise_operator}.
\end{definition}

Indeed, one can check that in this sense the scale-spaces of interest are isomorphic in the following way, akin to \Cref{res:semifield_isomorphism}.
\begin{proposition}[Scale-Space Isomorphisms] 
    \label{res:scale_space_isomorphisms} 
    ~
    \begin{itemize}
        \item The quadratic root scale-spaces over the root semifields \(R_p\) are isomorphic to the Gaussian scale-space over the nonnegative linear semifield \(L_{\geq 0}\).
    
        \item The quadratic logarithmic scale-spaces over the logarithmic semifields \(L_\mu\) are isomorphic to the Gaussian scale-space over the nonnegative linear semifield \(L_{\geq 0}\).

        \item The \(\alpha\)-dilation scale-space over the tropical max semifield \(\tropmax\) is isomorphic to the \(\alpha\)-erosion scale-space over the tropical min semifield \(\tropmin\).

        \item Informally, in the limit \(\mu \to \pm \infty\) the quadratic logarithmic scale-spaces over the logarithmic semifields \(L_\mu\) ``converge'' to the quadratic (\(\alpha=2\)) dilation and erosion scale-spaces of the tropical semifields \(\tropmaxmin\).
    \end{itemize}
\end{proposition}

\Cref{def:scale_space_isomorphism} also gives us a way to create new (isomorphic) semifield scale-spaces from existing ones in the following way.
Take any existing semifield scale-space \(\Phi_t\) over a semifield \(R\), and let \(\varphi: \tilde R \to R\) be any semifield isomorphism. 
We then simply define the new scale-space \(\tilde \Phi_t =  \PO_{\varphi^{-1}} \circ \Phi_t \circ \PO_\varphi\) over the semifield \(\tilde R\).

In \cite{florack2001nonlinear} Florack creates nonlinear scale-spaces in exactly this way by performing a pointwise transformation on the Gaussian scale-space and deducing what nonlinear PDE corresponds to the obtained evolution.

More specifically, let \(\varphi : R \to L\) be a monotonic twice continuously differentiable transformation, where \(R \subseteq \bbR\) is some subset of the reals.
We start with the isotropic diffusion PDE on \(\bbR^2\): 
\begin{equation}
        \pdv{f}{t} = \frac{1}{2} \Delta f.
\end{equation}
We then define a new evolution \(g(x, t) := \varphi^{-1}(f(x,t))\). 
Let us derive what PDE \(g\) obeys.
Using the shorthand \(\varphi(g) := \PO_\varphi(g)\) for a moment, it follows from the equalities:
\begin{align*}
    \pdv{g}{t} &= \pdv{}{t} \varphi^{-1}(f) = \frac{1}{\varphi'(g)} \ \pdv{f}{t}, \\
    \Delta f   &= \Delta(\varphi(g)) = \varphi''(g) \| \nabla g \|^2 + \varphi'(g) \Delta g,
\end{align*}
that the PDE that describes the evolution of \(g\) is:
\begin{equation}
    \pdv{g}{t} = \frac{\varphi''(g)}{\varphi'(g)} \frac{1}{2}\| \nabla g \|^2 +  \frac{1}{2} \Delta g.
\end{equation}
Florack suggests setting \(\mu := \varphi''/\varphi' = (\ln \varphi')'\) as a constant, as the class of non-trivial (that being non-affine) \(\varphi\)'s is then:
\begin{equation} 
    \varphi(x) = e^{\mu x} \text{ with } \mu \neq 0, \mu \in \bbR,
\end{equation}
up to affine transformations. 
This transformation of the PDE is known as the \textit{Cole-Hopf transformation} \cite[p.195]{evans2010partial}.
This is exactly the isomorphism between the logarithmic and linear semifield as seen in \Cref{res:semifield_isomorphism}, and indeed, with this \(\varphi\) we get the quadratic logarithmic scales spaces:
\begin{equation}
    \pdv{g}{t} = \mu\frac{1}{2} \|\nabla g \|^2 + \frac{1}{2} \Delta g.
\end{equation}
In the extreme cases of the transformation, that being \(\mu = \pm \infty\), the diffusion part becomes negligible in comparison to the erosion/dilation part, and one can say that the morphological scale-spaces arise.

If we instead choose $\varphi(x) = x^p$ the quadratic root scale-spaces arise:
\begin{equation}
    \pdv{g}{t} = \frac{p-1}{g}\ \frac{1}{2} \|\nabla g \|^2 + \frac{1}{2} \Delta g.
\end{equation}


\section{Consequences} \label{sec:consequences}

In this section we explore the consequences of the semifield scale-spaces axioms (\Cref{def:semifield_scale_space}).

We start by showing that the equivariance axioms of the scale-space representation lead to invariance properties of the kernel \(\kappa_t : \bbR^2 \times \bbR^2 \to R\).
For example, the translation equivariance of \(\Phi_t\) (Axiom \ref{ax:r2_equivariance}) implies that the kernel \(\kappa_t\) (Axiom \ref{ax:r2_linearity}) is translation invariant in the sense that \(\kappa_t(v + x, v + y) = \kappa_t(x,y)\) for all \(x,y,v \in \bbR^2\).

We then show that, due to the translation equivariance (Axiom \ref{ax:r2_equivariance}), the semifield scale-space can be written as a semifield convolution with a reduced kernel \(k_t : \bbR^2 \to R\).
From there on out we show \Cref{res:explicit_form_reduced_kernel} which gives an explicit form of the reduced kernel \(k_t\) in the semifield Fourier domain, this being the main theorem of the article.

\subsection{Equivariance of Operator becomes Invariance of Kernel}

In this subsection we show how the translation, rotoreflection, and scaling equivariance axioms on the operator family \(\Phi_t\) translate to corresponding invariances on the kernel \(\kappa_t\). 
The upcoming three lemmas are straightforward, generally known, and basically identical in proof.

\begin{lemma}[Translation Invariance] \label{res:translation_invariance}
    From the integral operator (Axiom \ref{ax:r2_linearity}) and translation equivariance (Axiom \ref{ax:r2_equivariance}) it follows that the kernel is translation invariant, that is:
    \begin{equation}
        \kappa_t(v + x,v + y) = \kappa_t(x,y),
    \end{equation}
    for all for all $x,y,v \in \mathbb{R}^2$ and $t>0$.
\end{lemma}
\begin{proof}
    We rewrite the translation equivariance (Axiom \ref{ax:r2_equivariance}) as
    \begin{equation}
        \TO_{-v} \circ \Phi_t \circ \TO_{v} = \Phi_t.
    \end{equation}
    We apply some dummy function \(f \in \cH\) and evaluate it at some dummy position \(x \in \bbR^2\):
    \begin{equation}
    \begin{split}
        ((\TO_{-v} \circ \Phi_t \circ \TO_{v})(f))(x) = (\Phi_t(f))(x).
    \end{split}
    \end{equation}
    Using the definition of the translation operator \(\TO_{v}\) \eqref{eq:translation_operator} and the integral operator axiom we expand this into:
    \begin{equation}
        \rint_{y \in \bbR^2} \kappa_t(v + x,y) \rt f(-v + y) = \rint_{y \in \bbR^2} \kappa_t(x,y) \rt f(y).
    \end{equation}
    Using the translation invariance property \eqref{eq:translation_integration} of the semifield integration gives:
    \begin{equation} \label{eq:before_fun_lemma}
        \rint_{y \in \bbR^2} \kappa_t(v + x,v + y) \rt f(y) = \rint_{y \in \bbR^2} \kappa_t(x,y) \rt f(y).
    \end{equation}
    Given that this should hold for all \(f \in \cH\) we can conclude:
    \begin{equation} \label{eq:after_fun_lemma}
        \kappa_t(v + x,v + y) = \kappa_t(x,y).
    \end{equation}
\end{proof}

{\if\arxiv0\color{RoyalBlue}\fi
The last step of the proof can be understood as a semifield version of \textit{Fundamental lemma of the calculus of variations}.
There are several versions of this lemma but generally they are of the form
\begin{equation}
    \forall g \ \int f g = 0 \Rightarrow f = 0,
\end{equation}
with some assumptions on the nature of \(f\) and \(g\).
Using the substitution \(f = f_1 - f_2\) we can equivalently write
\begin{equation} \label{eq:fun_lemma_variations}
    \forall g \ \int f_1 g = \int f_2 g \Rightarrow f_1 = f_2.
\end{equation}
The lemma in its second form \eqref{eq:fun_lemma_variations} is used in the last step of the proof of \Cref{res:translation_invariance}, specifically when transitioning from \eqref{eq:before_fun_lemma} to \eqref{eq:after_fun_lemma}. 
We will also apply it in the upcoming proofs of \Cref{res:rotoreflection_invariance} and \Cref{res:time_scale_invariance}.
The proof of the lemma is straightforward in our setting; the kernel \(\kappa_t\) is assumed to be continuous (Axiom \ref{ax:r2_linearity}), and the space \(\cH\) contains all indicator functions, meaning that a standard ``concentration'' argument works.
}

\begin{lemma}[Rotoreflection Invariance] \label{res:rotoreflection_invariance}
    From the integral operator (Axiom \ref{ax:r2_linearity}) and the rotoreflection equivariance (Axiom \ref{ax:r2_rotoreflection_equivariance}) it follows that the kernel is rotoreflection invariant, that is:
    \begin{equation}
        \kappa_t(Qx, Qy) = \kappa_t(x, y),
    \end{equation}
    for all orthonormal $Q \in \bbR^{2 \times 2}$, $x,y \in \mathbb{R}^2$, and $t>0$.
\end{lemma}
\begin{proof}
    We rewrite the rotoreflection equivariance (Axiom \ref{ax:r2_rotoreflection_equivariance}) as
    \begin{equation}
        \RO_{Q^{-1}} \circ \Phi_t \circ \RO_{Q} = \Phi_t.
    \end{equation}
    We apply some dummy function \(f \in \cH\) and evaluate it at some dummy position \(x \in \bbR^2\):
    \begin{equation}
    \begin{split}
        ((\RO_{Q^{-1}} \circ \Phi_t \circ \RO_{Q})(f))(x) = (\Phi_t(f))(x).
    \end{split}
    \end{equation}
    Using the definition of the rotoreflection operator \(\RO_{Q}\) \eqref{eq:rotoreflection_operator} and integral operator axiom we expand this to:
    \begin{equation}
        \rint_{y \in \bbR^2} \kappa_t(Qx, - y) \rt f(Q^{-1}y) = \rint_{y \in \bbR^2} \kappa_t(x,y) \rt f(y).
    \end{equation}
    Using the rotoreflection invariance property \eqref{eq:rotoreflection_integration} of the semifield integration gives:
    \begin{equation}
        \rint_{y \in \bbR^2} \kappa_t(Qx, Qy) \rt f(y) = \rint_{y \in \bbR^2} \kappa_t(x, y) \rt f(y).
    \end{equation}
    Given that this should hold for all \(f \in \cH\) we can conclude:
    \begin{equation}
        \kappa_t(Qx, Qy) = \kappa_t(x, y).
    \end{equation}
\end{proof}

\begin{lemma}[Scale Invariance] \label{res:time_scale_invariance}
    From the integral operator (Axiom \ref{ax:r2_linearity}) and scaling equivariance (Axiom \ref{ax:r2_scaling}) it follows that the kernel is scale invariant, that is:
    \begin{equation}
        \chi(s) 
        \rt \kappa_{s^\alpha t}(sx,sy) = \kappa_t(x,y).
    \end{equation}
    for all $x,y \in \mathbb{R}^2$ and $s,t>0$.
\end{lemma}
\begin{proof}
    We rewrite the scaling equivariance (Axiom \ref{ax:r2_scaling}) as
    \begin{equation}
        \SO_{1/s} \circ \Phi_{s^\alpha  t} \circ \SO_{s} = \Phi_{t}.
    \end{equation}
    We apply some dummy function \(f \in \cH\) and evaluate it at some dummy position \(x \in \bbR^2\):
    \begin{equation}
        ((\SO_{1/s} \circ \Phi_{s^\alpha  t} \circ \SO_{s})(f))(x) = (\Phi_t(f))(x).
    \end{equation}
    Using the definition of the scaling operator \(\SO_{s}\) \eqref{eq:scaling_operator} and integral operator axiom we expand this to:
    \begin{equation}
        \rint_{y \in \bbR^2} \kappa_{s^\alpha  t}(sx,y) \rt f(y/s) = \rint_{y \in \bbR^2} \kappa_t(x,y) \rt f(y).
    \end{equation}
    Using the scaling property \eqref{eq:scaling_integration} of the semifield integration gives:
    \begin{equation}
        \chi(s) \rt \rint_{y} \kappa_{s^\alpha  t}(sx,sy) \rt f(y) = \rint_{y} \kappa_t(x,y) \rt f(y).
    \end{equation}
    Given that this should hold for all \(f \in \cH\) we conclude:
    \begin{equation}
        \chi(s) \rt \kappa_{s^\alpha  t}(sx,sy) = \kappa_t(x,y).
    \end{equation}
\end{proof}


Consider now \Cref{res:translation_invariance}.
Because we can freely choose the translation \(v\), we can also choose \(v = -y\):
\begin{equation}
    \kappa_t(x,y) = \kappa_t(x - y,0) =: k_t(x-y).
\end{equation}
We thus see that \(\kappa_t\) is completely characterized by its behaviour on \(\kappa_t(\cdot, 0)\), which we define as the reduced kernel \(k_t : \bbR^2 \to R\). Plugging this newfound knowledge back into the integral operator axiom we get the following result.
    
\begin{lemma}[Translation Equivariance implies Semifield Convolution] \label{res:translation_invariance_convolution}
    Consider the integral operator (Axiom \ref{ax:r2_linearity}) and translation equivariance (Axiom \ref{ax:r2_equivariance}).
    Define the reduced kernel \(k_t(x) = \kappa_t(x,0)\). 
    We can write the scale-space operator \(\Phi_t\) as a semifield convolution:
    \begin{equation}
        \Phi_t f = k_t \rconv f.
    \end{equation}
\end{lemma}
\begin{proof}
    This follows immediately from \Cref{res:translation_invariance}.
    \begin{equation}
    \begin{split}
        (\Phi_t f)(x) 
        &= \rint_{y \in \bbR^2} \kappa_t(x,y) \rt f(y)\\ 
        &= \rint_{y \in \bbR^2} \kappa_t(x-y,0) \rt f(y) \\
        &= \rint_{y \in \bbR^2} k_t(- y + x) \rt f(y) \\
        &= (k_t \rconv f)(x).
    \end{split}
    \end{equation}
\end{proof}

\begin{lemma}[Convolution Property of Reduced Kernel] \label{res:convolution_property_kernel}
    From the integral operator (Axiom \ref{ax:r2_linearity}), the one-parameter semigroup (Axiom \ref{ax:r2_one_parameter_semigroup}), the strong continuity (Axiom \ref{ax:r2_strong_continuity}), and the translation equivariance (Axiom \ref{ax:r2_equivariance}), it follows that the reduced kernel $k_t(x) = \kappa_t(x,0)$ satisfies:
    \begin{equation}
        k_s \rconv k_t = k_{s+t} \text{ for all } t,s > 0.
    \end{equation}
\end{lemma}
\begin{proof}
    We have already seen in \Cref{res:translation_invariance_convolution} that the integral operator axiom and the translation equivariance axiom imply that
    \begin{equation}
        \Phi_t f = k_t \rconv f.
    \end{equation}
    If we use this formula together with the one-parameter semigroup property axiom we get
    \begin{equation}
        k_s \rconv (k_t \rconv f) = k_{s+t} \rconv f.
    \end{equation}
    Using associativity of semifield convolution \eqref{eq:semifield_convolution_associativity} on the l.h.s.:
    \begin{equation}
        (k_s \rconv k_t) \rconv f = k_{s+t} \rconv f.
    \end{equation}
    We are free to choose \(f = k_\varepsilon\) for \(\varepsilon>0\):
    \begin{equation}
        (k_s \rconv k_t) \rconv k_\varepsilon = k_{s+t} \rconv k_\varepsilon.
    \end{equation}
    From the strong continuity axiom we know that \(\lim_{\varepsilon \to 0} k_\varepsilon \rconv k_t = k_t\), thus, after taking this limit, we can conclude:
    \begin{equation}
        k_s \rconv k_t = k_{s+t}.
    \end{equation}
\end{proof}

\subsection{Towards the Semifield Fourier Domain} \label{sec:going_to_fourier_domain}

We see that to move forward we need a way to efficiently work with semifield convolutions.
The semifield Fourier transform (\Cref{def:semifield_fourier_transform}) has the important property of turning convolutions into much more wieldy pointwise multiplication.
This is why we translate all previous lemmas to the semifield Fourier domain using the semifield Fourier transform.

\begin{lemma} \label{res:properties_fourier_kernel}
    Consider all axioms of a semifield scale-space (\Cref{def:semifield_scale_space}). 
    The reduced kernel $k_t(x) := \kappa_t(\cdot,0)$ in the semifield Fourier domain $\hat{k}_t = \cF_R(k_t)$ satisfies
    \begin{align}
        \label{eq:rotoreflection_fourier_kernel}
        &\hat{k}_t(Q\omega) = \hat{k}_{t}(\omega),\\
        \label{eq:time_scale_fourier_kernel}
        &\hat{k}_{s^\alpha  t}(\omega/s) = \hat{k}_t(\omega),\\
        \label{eq:semifield_fourier_kernel}
        &\hat{k}_s(\omega) \rt \hat{k}_t(\omega) = \hat{k}_{s+t}(\omega),
    \end{align}
    for all orthonormal $Q\in \bbR^{2 \times 2}$, $\omega \in \bbR^2$, and $s,t>0$.
\end{lemma}
\begin{proof}
    \Cref{res:rotoreflection_invariance} tells us that
    \begin{equation}
        \kappa_t(Qx, Qy) = \kappa_t(x, y).
    \end{equation}
    Translating this to the reduced kernel \(k_t\) gives
    \begin{equation}
        k_t(Qx) = k_t(x).
    \end{equation}
    Taking the semifield Fourier transform on both sides, and using the rotoreflection equivariance property of the Fourier transform \eqref{eq:rotoreflection_fourier}, we get:
    \begin{equation}
        \hat{k}_t(Q \omega) = \hat{k}_t(\omega).
    \end{equation}
    \Cref{res:time_scale_invariance} also tells us that
    \begin{equation}
        \chi(s) \rt \kappa_{s^\alpha  t}(sx,sy) = \kappa_t(x,y).
    \end{equation}
    Translating this to the reduced kernel \(k_t\) gives
    \begin{equation}
        \chi(s) \rt k_{s^\alpha  t}(sx) = k_t(x).
    \end{equation}
    Taking the semifield Fourier transform on both sides, and using the scaling equivariance property of the Fourier transform \eqref{eq:scaling_fourier}, we get:
    \begin{equation}
        \chi(s) \rt \chi(1/s) \rt \hat{k}_{s^\alpha  t}(\omega/s) = \hat{k}_t(\omega).
    \end{equation}
    Using that \(\chi\) is a homomorphism we know that \(\chi(s) \rt \chi(1/s) = \chi(s/s) = \chi(1) = \ro\), so we can simplify this to 
    \begin{equation}
        \hat{k}_{s^\alpha  t}(\omega/s) = \hat{k}_t(\omega).
    \end{equation}
    \Cref{res:convolution_property_kernel} tells us that
    \begin{equation}
        k_s \rconv k_t = k_{s+t}.
    \end{equation}
    Taking the semifield Fourier transform on both sides, and using the convolution property of the Fourier transform \eqref{eq:convolution_fourier}, we get:
    \begin{equation}
        \hat{k}_s \rt \hat{k}_t = \hat{k}_{s+t}.
    \end{equation}
\end{proof}

Let us check if the reduced kernels of the scale-spaces of interest indeed satisfy the properties listed in the previous lemma by inspecting their Fourier transforms. 

\begin{proposition}[Semifield Fourier Transform of Kernels of Interest] \label{res:fourier_kernels}
    Consider the employed semifield Fourier transforms (\Cref{def:employed_fourier}) and the semifield scale-spaces of interest (\Cref{def:scale_spaces_interest}). 
    Let $k_t(x) := \kappa_t(\cdot,0)$ be the reduced kernel and \(\hat{k}_t = \cF_R(k_t)\) its semifield Fourier transform.
    \begin{enumerate}[label=\alph*)]
        \item For the Gaussian scale-space over the linear semifield \(L\) we have:
        \begin{equation}
            \hat{k}^{L}_t(\omega) = \exp \Par{-\frac{1}{2} t \|\omega\|^2}.
        \end{equation}

        \item For the quadratic root scale-spaces over the root semifields \(R_p\) we have:
        \begin{equation}
            \hat{k}^{R_p}_t(\omega) = \exp \Par{-\frac{1}{2p} t \|\omega\|^2}.
        \end{equation}

        \item For the quadratic logarithmic scale-spaces over the logarithm semifields \(L_\mu\) we have:
        \begin{equation}
            \hat{k}^{L_\mu}_t(\omega) = -\frac{1}{2 \mu } t \|\omega\|^2.
        \end{equation}
        
        \item For the \(\alpha\)-dilation scale-space over the tropical max semifield \(\tropmax\) we have:
        \begin{equation}
            \hat{k}^{\tropmax}_t(\omega) = \frac{1}{\alpha} t \|\omega\|^\alpha.
        \end{equation}
        
        \item For the \(\alpha\)-erosion scale-space over the tropical min semifield \(\tropmin\) we have:
        \begin{equation}
            \hat{k}^{\tropmin}_t(\omega) = -\frac{1}{\alpha} t \|\omega\|^\alpha.
        \end{equation}
    \end{enumerate}
\end{proposition}

\subsection{Explicit Form of the Reduced Kernel}

With the results derived above we are now ready to derive the explicit form of the reduced kernel in the Fourier domain $\hat{k}_t$, and, in turn, an expression for the reduced kernel \(k_t\). 
But to succinctly state this explicit form we need one extra ingredient: semifield exponentiation, i.e. a generalization of repeated semifield multiplication.

\begin{definition}[Semifield Exponentiation] \label{def:semifield_exponentiation}
    Let \(R=(R,\rz,\ro,\rp,\rt)\) be a one-dimensional metric semifield.
    The semifield exponentiation \(\exp_R: \bbR \to R\) is defined as the (up to time scaling unique\footnote{The connected part of \((R,\rt)\) that contains \(\ro\) is a one-dimensional Lie group. Semifield exponentiation is the Lie group exponential and is determined by a one-dimensional tangent vector at the identity, thus giving us the up to scaling uniqueness.}) mapping that satisfies:
    \begin{itemize}[label={}]
        \item \(\exp_R(s) \otimes \exp_R(t) = \exp_R(s+t)\) for all \(s,t \in \bbR\),
        \item \(\exp_R(0) = \ro\),
        \item \(\lim_{t \to \infty} \exp_R(t) = \rz\).
    \end{itemize}
\end{definition}

To distinguish the semifield exponentiation from regular exponentiation we always indicate the former with the semifield in the subscript, and the latter without any subscript.

\begin{definition}[Employed Semifield Exponentiation]
    ~
    \begin{enumerate}[label=\alph*)]
        \item In the (nonnegative) linear semifield \(L\) case we have \(\exp_{L}(t) = \exp(-t)\).
        \item In the root semifields \(R_p\) case we have \(\exp_{R_p}(t) = \exp(-t)\).
        \item In the logarithmic semifields \(L_\mu\) case we have \(\exp_{L_\mu}(t) = - \sign(\mu) t\).
        \item In the tropical max semifield \(\tropmax\) case we have \(\exp_{\tropmax}(t) = -t\).
        \item In the tropical min semifield \(\tropmin\) case we have \(\exp_{\tropmin}(t) = t\).
    \end{enumerate}
\end{definition}

\begin{theorem}[Explicit form Reduced Kernel] \label{res:explicit_form_reduced_kernel}
    Let \(R\) be a one-dimensional metric semifield, \(\cF_R\) the semifield Fourier transform (\Cref{def:semifield_fourier_transform}), and \(\exp_R\) the semifield exponentiation (\Cref{def:semifield_exponentiation}).
    Consider all axioms of a semifield scale-space (\Cref{def:semifield_scale_space}).
    We have that the reduced kernel \(k_t\) is (up to a time scaling) equal to:
    \begin{equation}
    \begin{split}
        \hat{k}_t(\omega) &= \exp_R (\|\omega\|^\alpha t),\\
        k_t(x) &= (\cF_R^{-1} \hat{k}_t)(x),
    \end{split}
    \end{equation}
    where \(\alpha\) is the scaling power (Axiom \ref{ax:r2_scaling}).
\end{theorem}

Before we continue with the proof, we can check that, indeed, all semifield Fourier transforms of the reduced kernels, as listed in \Cref{res:fourier_kernels}, have this stated form.
    
\begin{proof}
    Consider the reduced kernel \(\hat{k}_t\) in the Fourier domain and all its properties as listed in \Cref{res:properties_fourier_kernel}.
    Due to the rotoreflectional symmetry \eqref{eq:rotoreflection_fourier_kernel} we will abuse notation slightly and write
    \begin{equation}
        \hat{k}_t(\omega) = \hat{k}_t(\|\omega\|) = \hat{k}_t(r),
    \end{equation}
    where \(r = \|\omega\|\). 
    Taking \(s = r  > 0\) in the scaling invariance \eqref{eq:time_scale_fourier_kernel} 
    we get
    \begin{equation}
        \hat{k}_t(r) = 
        \hat{k}_{r^\alpha  t}(1) \text{ for } r > 0.
    \end{equation}
    Due to the one-parameter semigroup property \eqref{eq:semifield_fourier_kernel} and $t \mapsto \hat{k}_t(\omega)$ being continuous, we have (up to a time scaling)
    \begin{equation}
        \hat{k}_t(1) = \exp_R(t) \text{ or } \hat{k}_t(1) = \rz.
    \end{equation}
    We are not interested in the \(\hat{k}_t(1)=\rz\) case as it would imply, together with the previous equation, that \(\hat{k}_t\) is identically zero, corresponding to a non-relevant scale-space.
    Combining the equations found so far we get
    \begin{equation}
        \hat{k}_t(r) = \exp_R (r^\alpha  t) \text{ for } r > 0.
    \end{equation}
    which we can extend to \(r=0\) by continuity.
    Taking the inverse semifield Fourier transform concludes the proof.
\end{proof}

The up-to-a-time-scaling non-uniqueness is something we can \textit{not} avoid as every semifield scale-space \(\Phi_t\) corresponds to an infinite family of scale-spaces \(\tilde{\Phi}_t = \Phi_{s t}\) for every \(s > 0\).

The above theorem shows that every (one dimensional metric) semifield \(R\) corresponds to a unique one-parameter family of semifield scale-spaces, where the scaling power \(\alpha\) acts as the parameter.

There is one caveat here though, and that is that not every \(\alpha\) necessarily results in a \(\tilde k_t\) which is in the domain of the used inverse semifield Fourier transform.
For example, in the tropical semifields the case \(\alpha=1\) results in a \(\tilde k_t\) which we cannot insert into the inverse transforms listed in \Cref{def:employed_fourier}.

{\if\arxiv0\color{RoyalBlue}\fi
\section{Architecture} \label{sec:architecture}
}

In this section, we will briefly discuss the PDE-CNN architecture by defining the PDE sublayers corresponding to semifield scale-spaces \eqref{eq:pde_sublayer}, the PDE sublayer corresponding to convection \eqref{eq:convection_pde_sublayer}, and the affine sublayer \eqref{eq:affine_sublayer}. 
We also illustrate how these sublayers combine to form a PDE layer, and how multiple PDE layers come together to construct a PDE-CNN in \Cref{fig:pde_cnn_architecture_crop}.
 
Consider any one-dimensional metric semifield \(R\) and a corresponding semifield scale-space \(\Phi_t\). 
As \Cref{res:translation_invariance_convolution} shows, we can write the scale-space \(\Phi_t f\) of a (appropriate) function \(f : \bbR^2 \to R\) as \(\Phi_t f = k_t \rconv f\) where \(k_t : \bbR^2 \to R \) is the reduced kernel, and \(\rconv\) is the semifield convolution.
We want to implement these semifield scale-spaces to use them within the design of our PDE-CNNs.
However, in practice, we cannot work with general signals defined on the continuum of \(\bbR^2\), and we need to discretize our setting.


As is usual in machine learning we imagine images \(f : \bbR^2 \to R\) as sampled images \(\tilde f : Z \to R\) on a grid \(Z \subset \bbR^2\) such that \(\tilde f(z) = f(z)\) for all grid points \(z \in Z\).
We idealize the grid as the infinite integer grid \(\bbZ^2\) here for the sake of simplicity (we have no boundary concerns).
Let \(\tilde f, \tilde k : \bbZ^2 \to R\) be the discretized versions of an image \(f\) and any kernel \(k\).
The discrete semifield convolution is defined as
\begin{equation}
    (\tilde k \rconv \tilde f)\Bra{i,j} = \bigoplus_{{m,n} \in \bbZ^2} \tilde k \Bra{-m + i,- n + j} \rt \tilde f\Bra{m,n},
\end{equation}
where we have used the notation \(\Bra{\cdot,\cdot}\) to emphasize the discrete nature.

With the discrete semifield convolution we can write down the formula for a PDE sublayer in the PDE-CNN.
The input of a PDE sublayer consists of signals \(\tilde f_i : \bbZ^2 \to R\) and matrices \(H_i \in \bbR^{2 \times 2}\) with \(i = 1, \dots, C\), and \(C\) being the amount of channels.
The matrices \(H_i\) act as the learnable parameters of the layer.
We consider the continuous scale-space kernel \(k_t : \bbR^2 \to R\) and create the discretized kernels \(\tilde k_i : \bbZ^2 \to R\) by defining 
\begin{equation} 
    \tilde k_i(x) = k_1(H_i x).
\end{equation}
Without loss of generality we may take \(t=1\) in our scale-space kernel \(k_t\) as a scaling in \(t\) can be captured in \(H_i\).
We then perform the discrete semifield convolutions to acquire our outputs \(\tilde g_i : \bbZ^2 \to R\):
\begin{equation} 
\label{eq:pde_sublayer}
\begin{split}
    &\text{PDE sublayer: } \\
    &\tilde g_i = \tilde k_i \rconv \tilde f_i.
\end{split}
\end{equation}
Note that every input channel is only convolved with a single kernel, this is also known as a \textit{depthwise convolution}.

{\if\arxiv0\color{RoyalBlue}\fi
The matrices \(H_i\) require some explanation.
As already touched upon in \Cref{sec:axioms}, in the semifield scale-space axioms we implicitly assumed the standard inner product \(\cG(x,y) = x^\top y\) on \(\bbR^2\), but this is not the only one we can choose. 
By choosing the inner product to be \[\cG(x,y) = x^\top G y,\quad G = H^\top H\] for any matrix \(H\) we get a scale-space representation that is complete identical, albeit ``stretched'' with respect to the coordinates.
The Gram matrix \(G=H^\top H\) is by construction symmetric positive definite (SPD) as required, and relieves us from coding a SPD constraint.
By considering kernels of the form \(k_t(Hx)\) we effectively include the possibility of processing the image with a different inner product.
The ``stretching'' induced by the matrix \(H\) in the Gaussian scale-space case corresponds to the \textit{affine Gaussian scale-space} studied in \cite{lindeberg1997shape}.
}

The PDE sublayer, as described in formula \eqref{eq:pde_sublayer}, seems to have nothing to do with a PDE at first glance.
However, remember that every semifield scale-space \(\Phi_t\) over a semifield \(R\) can be associated with a PDE, and the PDE sublayer is effectively a solver for the corresponding initial value problem.
The examples in \Cref{def:scale_spaces_interest} clarify this.

Alongside the PDE sublayers that correspond to semifield scale-spaces we also have the convection PDE sublayer.
The convection sublayer effectively solves the convection PDE by translating images.
The input of a convection PDE sublayer consists of images \(\tilde f_i : \bbZ^2 \to R\) and vectors \(v_i \in \bbR^{2}\) with \(i = 1, \dots, C\), and \(C\) being the amount of channels.
The vectors \(v_i\) act as the learnable parameters of the layer.
The output signals \(\tilde g_i : \bbZ^2 \to R\) are obtained through bilinear interpolating (Interp) the inputs at the appropriate positions:
\begin{equation} 
\label{eq:convection_pde_sublayer}
\begin{split}
    &\text{Convection sublayer: } \\
    &\tilde g_i[m,n] = \text{Interp}(\tilde f_i, m - (v_i)_1, n - (v_i)_2).
\end{split}
\end{equation}

The final ingredient we need is the affine layer which is defined as follows.
The input consists of channels \(\tilde f_i : \bbZ^2 \to \bbR\), weights \(w_{ij} \in \bbR\), and biases \(b_j \in \bbR\), with \(i = 1, \dots, C_i\), \(j = 1, \dots, C_o\), and \(C_i\), \(C_o\) being the amount of input and output channels respectively.
The weights \(w_{ij}\) and biases \(b_j\) act as the learnable parameters of the layer.
We then perform the following pointwise computation to get the \(C_o\) outputs \(\tilde g_j : \bbZ^2 \to R\)
\begin{equation} \label{eq:affine_sublayer}
\begin{split}
    &\text{Affine layer: } \\
    &\tilde g_j = b_j + \sum_{i = 1}^{C_i} w_{ij} \tilde f_i.
\end{split}
\end{equation}

By concatenating various PDE sublayers, that being either a sublayer that corresponds to a scale-space or a convection sublayer, with an affine combination layer at the end we form a PDE layer.
Multiple PDE layers after each other with an affine layer at the start and end creates a PDE-CNN.
The architecture is illustrated in \Cref{fig:pde_cnn_architecture_crop}.

\begin{figure}
    \centering
    \includegraphics[width=\linewidth]{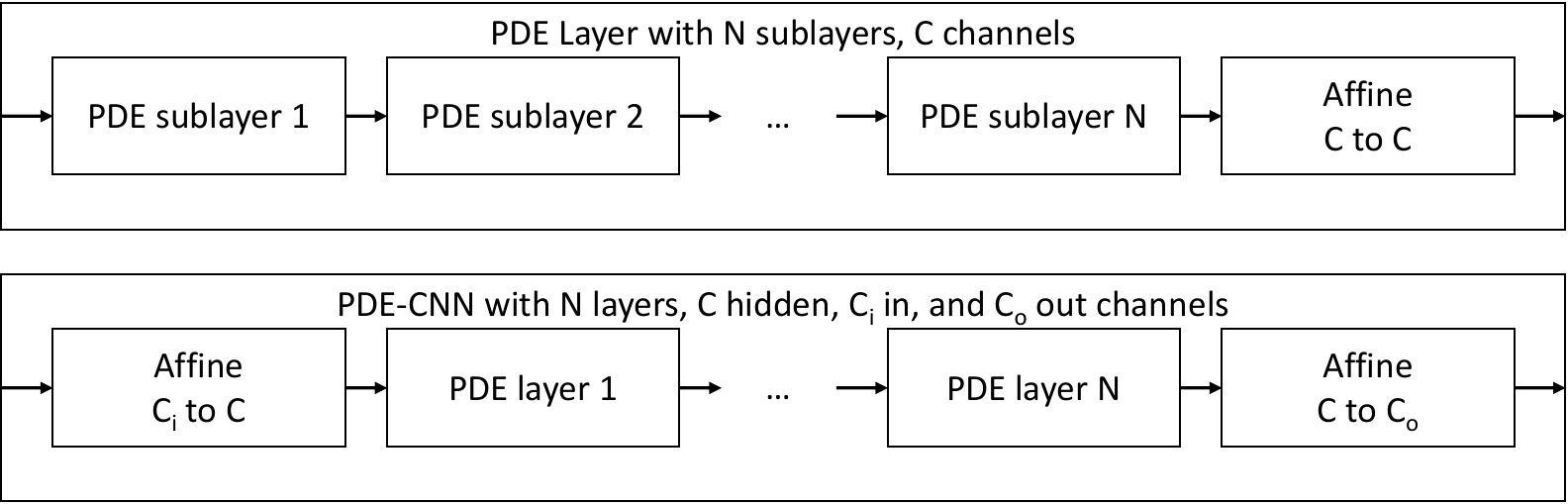}%
    \caption{The architecture of a \(N\)-layer PDE-CNN with \(C\) channels. A PDE sublayer is either of the form \eqref{eq:pde_sublayer} or \eqref{eq:convection_pde_sublayer}. }
    \label{fig:pde_cnn_architecture_crop}
\end{figure}

\section{Experiments} \label{sec:experiments}

\subsection{Including New PDEs in PDE-CNNs} \label{sec:experiment_semifields}

{\if\arxiv0\color{RoyalBlue}\fi
Current PDE-based neural networks employ three PDEs that generate scale-spaces: diffusion, dilation, and erosion \eqref{eq:pdes_pde_g_cnn}, which correspond respectively to the linear semifield \(L\), the tropical min semifield \(\tropmin\), and the tropical max semifield \(\tropmax\).
Our theory reveals at least two PDEs that have not yet been used in PDE-based neural networks: the PDE that generates the root scale-space \eqref{eq:quadratic_root_scale_space_pde} and the PDE that generates the logarithmic scale-space \eqref{eq:quadratic_log_scale_space_pde}, which arise naturally from the root semifield \(R_p\) and logarithm semifield \(L_\mu\).
Our first experiment will examine how the inclusion of these PDEs affects the accuracy of the PDE-CNN architecture, see \Cref{fig:pde_cnn_architecture_crop}.
}

The networks that we consider always include the convection PDE sublayer at the start of the PDE layer.
The PDE-CNNs will consist of 6 PDE layers, 32 channels, and have an average parameter count of approximately \(9\,500\), which changes with the amount of PDEs we add to its PDE layers.

We will test the networks on the DRIVE dataset \cite{staal2004ridge}.
The dataset consists of fundus images, with the goal being vessel segmentation.
In \Cref{fig:drive_example} one can see an example of such an image and its segmentation.
{\if\arxiv0\color{RoyalBlue}\fi
We selected the DRIVE dataset because its images contain vessels at varying scales, making it well-suited for applying scale-space techniques, including PDE-based neural networks.
It also allows for comparison with the results in \cite{smets2023pde,bellaard2023analysis,bellaard2023geometric,pai2023functional}, which also use the DRIVE dataset.
}

The dataset consist of a training set of \(20\) images and a test set of \(20\) images.
All images are \(584\times565\) pixels in 8-bit RGB color, which we rescale to the \([0,1]\) range by dividing by \(255\).
We divided the training set into \(2\,880\) overlapping patches of \(64 \times 64\).
Patches that contain no annotation, i.e. patches that are essentially completely within the black mask (\Cref{fig:drive_example}), are removed, leaving us with \(2\,409\) patches.

\begin{figure}
    \centering
    \begin{subfigure}[t]{0.45\linewidth}
        \centering
        \includegraphics[width=0.8\linewidth]{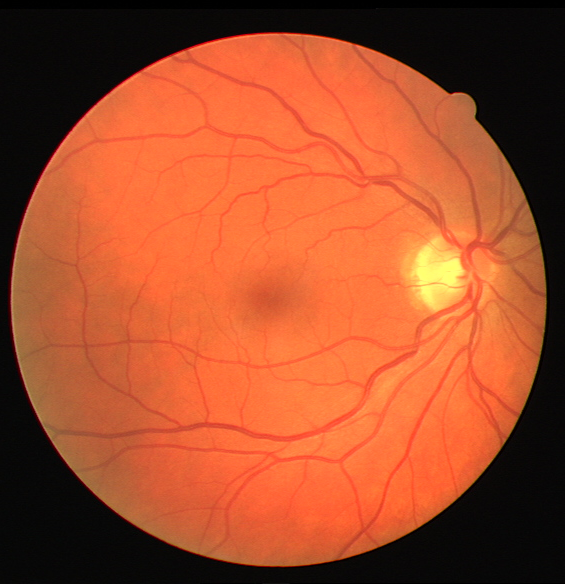}%
        \caption{Input image of the fundus of the eye.}
    \end{subfigure}\hspace{1em}
    \begin{subfigure}[t]{0.45\linewidth}
        \centering
        \includegraphics[width=0.8\linewidth]{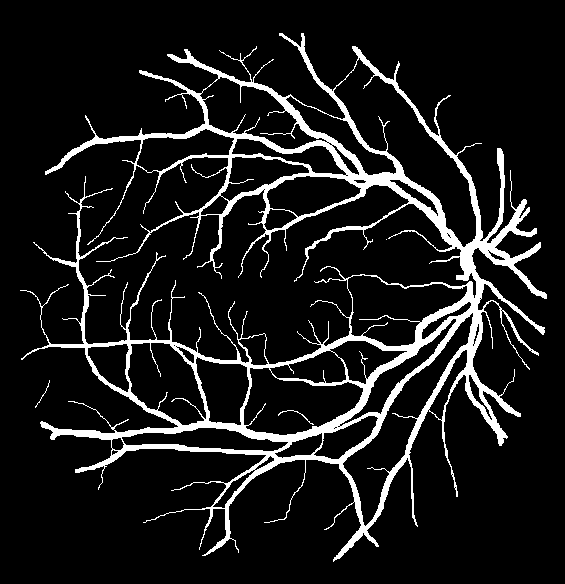}%
        \caption{Target vessel segmentation.}
    \end{subfigure}\\
    \begin{subfigure}[t]{0.45\linewidth}
        \centering
        \includegraphics[width=0.8\linewidth]{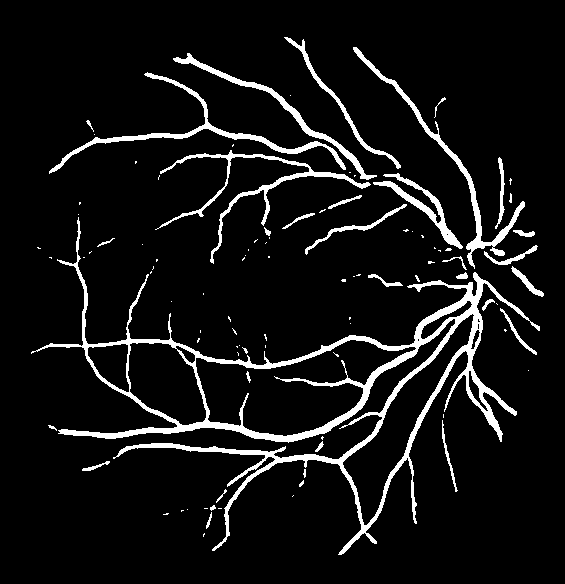}%
        \caption{Example output of a network with a dice coefficient of \(\approx 0.78\). Notice how small vessels are missed.}
    \end{subfigure}\hspace{1em}
    \begin{subfigure}[t]{0.45\linewidth}
        \centering
        \includegraphics[width=0.8\linewidth]{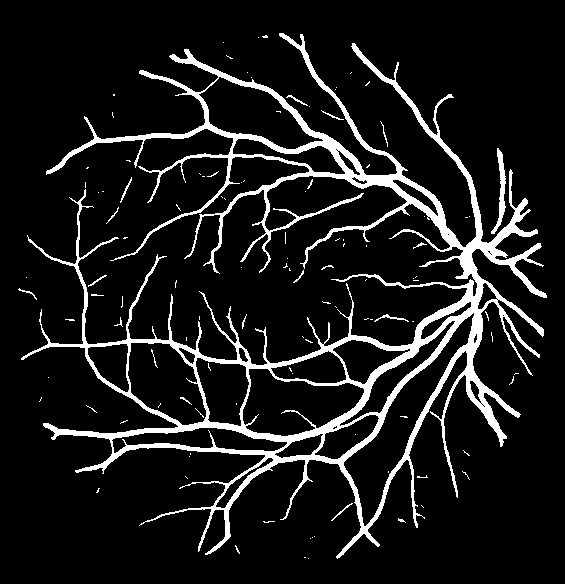}%
        \caption{Example output of a network with a dice coefficient of \(\approx 0.81\). Notice how more small vessels are captured.}
    \end{subfigure}
    \caption{One instance of an input and its corresponding target segmentation from the DRIVE dataset, together with two example outputs of networks with different dice coefficients.}
    \label{fig:drive_example}
\end{figure}

We train on batches consisting of \(8\) patches.
We empirically found that a higher batch size results in a worse test set accuracy.
We use the AdamW optimizer with an initial learning rate of \(0.01\) that decays linearly to \(0.001\) over the first \(1\,000\) batches.
The beta, epsilon and weight decay parameters of AdamW are kept at their (PyTorch) default values of \((0.9, 0.999)\), \(10^{-8}\), and \(0.01\).
During training we keep track of the Dice coefficient on the test set and the best one is stored.
We train until the Dice coefficient on the test set no longer increases, which happens within \(20\,000\) batches.
We then repeat this 5 times for every possible situation.

We empirically found that adding the linear semifield \(L\), that is we add a PDE sublayer corresponding to the Gaussian scale-space, does \textit{not} affect the accuracy of any the networks.
This can be explained by noticing that such a PDE sublayer can be emulated completely and effectively by the convection sublayer \eqref{eq:convection_pde_sublayer} together with the affine sublayer \eqref{eq:affine_sublayer} , which are always components of the PDE-CNNs we consider here. 
For this reason we have omitted the linear semifield PDE sublayers altogether.
This is in agreement with the results found in \cite[p.28]{castella2021introduction}.

The result can be found in \Cref{fig:different_semifields}.
We observe multiple things:
\begin{itemize}
    \item Adding semifields to the existing PDE-CNN architecture, which only employ the tropical semifields and convection, may enhance accuracy, albeit not significantly, as observed in the case of going from \(\{\tropmax,\tropmin\}\) to \(\{\tropmax,\tropmin,L_\mu\}\).
    \item The inclusion of the tropical min semifield \(\tropmin\) \textit{always} increases accuracy, most starkly seen when going from \(\{L_\mu\}\) to \(\{\tropmin,L_\mu\}\).
    \item Adding semifields does \textit{not} necessarily improve accuracy, as is evident from the last row when compared to the two-semifield models in the middle rows.
    \item The inclusion of the root semifield \(R_p\) seems to make the training less stable, as indicated by the increase in spread within the scatter plot at the respective rows.
\end{itemize}
It is worth mentioning however that these results might be specific to the DRIVE dataset.


\definecolor{mygreen}{RGB}{54, 176, 60}
\definecolor{myred}{RGB}{176, 74, 110}
\newcommand{\cmark}{{\color{mygreen} \ding{51}}}%
\newcommand{\xmark}{{\color{myred} \ding{55}}}%

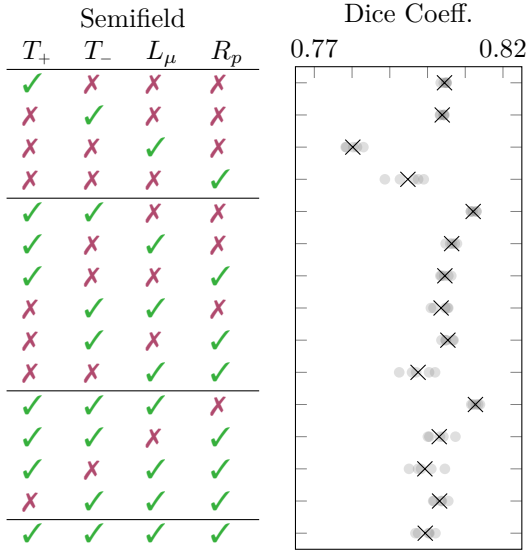
\begin{figure}
    \centering
    \begin{tabular}{llll}
        \multicolumn{4}{c}{Semifield} \\
        \(\tropmax\) & \(\tropmin\) & \(L_\mu\) & \(R_p\)  \\
        \hline
        \cmark & \xmark & \xmark & \xmark \\
        \xmark & \cmark & \xmark & \xmark \\
        \xmark & \xmark & \cmark & \xmark \\
        \xmark & \xmark & \xmark & \cmark \\
        \hline
        \cmark & \cmark & \xmark & \xmark \\
        \cmark & \xmark & \cmark & \xmark \\
        \cmark & \xmark & \xmark & \cmark \\
        \xmark & \cmark & \cmark & \xmark \\
        \xmark & \cmark & \xmark & \cmark \\
        \xmark & \xmark & \cmark & \cmark \\
        \hline
        \cmark & \cmark & \cmark & \xmark \\
        \cmark & \cmark & \xmark & \cmark \\
        \cmark & \xmark & \cmark & \cmark \\
        \xmark & \cmark & \cmark & \cmark \\
        \hline
        \cmark & \cmark & \cmark & \cmark \\
    \end{tabular}
    \raisebox{-0.477\height}{
    \begin{tikzpicture}
        \begin{axis}[
        width = 0.6\linewidth,
        height = 1.05\linewidth,
        ytick = {1,2,3,4,5,6,7,8,9,10,11,12,13,14,15},
        yticklabels = {,,,,,,,,,,,,,,},
        ymin = 0.5,
        ymax = 15.5,
        xmin = 0.765,
        xmax = 0.825,
        xtick = {0.77,0.78,0.79,0.80,0.81,0.82},
        xticklabels = {0.77,,,,,0.82},
        xlabel = {Dice Coeff.},
        xtick pos=upper,
        xticklabel pos=upper
    ]

        \addplot[
            only marks, 
            fill=gray, 
            fill opacity=0.25, 
            draw opacity=0
        ] coordinates {
            (0.8051, 15)
            (0.8049, 15)
            (0.8050, 15)
            (0.8035, 15)
            (0.8043, 15)
            (0.8034, 14)
            (0.8045, 14)
            (0.8033, 14)
            (0.8043, 14)
            (0.8043, 14)
            (0.7830, 13)
            (0.7813, 13)
            (0.7803, 13)
            (0.7785, 13)
            (0.7783, 13)
            (0.7927, 12)
            (0.7990, 12)
            (0.7975, 12)
            (0.7887, 12)
            (0.7964, 12)
            
            (0.8130, 11)
            (0.8129, 11)
            (0.8120, 11)
            (0.8115, 11)
            (0.8115, 11)
            (0.8078, 10)
            (0.8072, 10)
            (0.8065, 10)
            (0.8060, 10)
            (0.8048, 10)
            (0.8063, 9)
            (0.8055, 9)
            (0.8044, 9)
            (0.8037, 9)
            (0.8034, 9)
            (0.8055, 8)
            (0.8053, 8)
            (0.8047, 8)
            (0.8016, 8)
            (0.8009, 8)
            (0.8069, 7)
            (0.8067, 7)
            (0.8049, 7)
            (0.8048, 7)
            (0.8037, 7)
            (0.8020, 6)
            (0.8004, 6)
            (0.7970, 6)
            (0.7959, 6)
            (0.7925, 6)
            
            (0.8139, 5)
            (0.8134, 5)
            (0.8128, 5)
            (0.8120, 5)
            (0.8116, 5)
            (0.8074, 4)
            (0.8042, 4)
            (0.8033, 4)
            (0.8005, 4)
            (0.8001, 4)
            (0.8046, 3)
            (0.8010, 3)
            (0.7983, 3)
            (0.7975, 3)
            (0.7951, 3)
            (0.8055, 2)
            (0.8038, 2)
            (0.8036, 2)
            (0.8017, 2)
            (0.8015, 2)
            
            (0.8021, 1)
            (0.8005, 1)
            (0.8003, 1)
            (0.7975, 1)
            (0.7968, 1)
        };
        
        \addplot[
            only marks,
            color = black,
            mark = x, 
            dashed,
            mark options={solid},
            mark size=4
        ] coordinates {
            (0.8045, 15)
            (0.8039, 14)
            (0.7802, 13)
            (0.7948, 12)
            
            (0.8121, 11)
            (0.8064, 10)
            (0.8046, 9)
            (0.8036, 8)
            (0.8054, 7)
            (0.7975, 6)
            
            (0.8127, 5)
            (0.8031, 4)
            (0.7993, 3)
            (0.8032, 2)
            
            (0.7994, 1)
        };
    \end{axis}
    \end{tikzpicture}}
    \caption{A scatterplot of the accuracy of a 6-layer PDE-CNN on the DRIVE dataset, with various designs of the PDE layer as indicated in the table on the left. The crosses indicate the mean. The rows are organized according to the amount of semifields included in the model.}
    \label{fig:different_semifields}
\end{figure}

\subsection{Data Efficiency of PDE-CNNs} \label{sec:data_efficiency}

The data efficiency of PDE-G-CNNs on \(\bbM_2\) is already verified in \cite{pai2023functional}, but whether this desirable property holds in the PDE-CNN case is still left untested.
Our second experiment is therefore testing the data efficiency of a PDE-CNN on the DRIVE dataset.
The PDE-CNN we consider here employs three PDEs within its PDE layers: convection, dilation, and erosion, just as in the papers \cite{smets2023pde,bellaard2023analysis,bellaard2023geometric,pai2023functional}.

{\if\arxiv0\color{RoyalBlue}\fi
As a baseline we consider a CNN with \(31\, 488\) parameters, and compare this against a PDE-CNN with \(5\,280\) parameters.
To make the comparison fair both networks have \(6\) layers and \(24\) channels.
The only difference between the CNN and PDE-CNN is the kind of layer that is used.
In the CNN we use a standard 2D convolutional module with \(3\times3\) kernels together with a nonlinear activation function.
In the PDE-CNN we use a PDE layer as described in \Cref{fig:pde_cnn_architecture_crop}.
The size of the networks has been chosen this way such that both give a satisfactory Dice coefficient of \(\gtrapprox 0.80\) on the test set when trained on the complete training set.
}


Following the method in \cite{pai2023functional}, we randomly take \(1\%\) to \(100\%\) of the training data.
Other than that our methodology is identical to \Cref{sec:experiment_semifields}.

The result can be found in \Cref{fig:data_efficiency}.
We see that on the DRIVE dataset, in comparison with a standard CNN, the PDE-CNN not only features fewer parameters but also showcases competitive accuracy and increased data efficiency.
This mirrors the results found in \cite{pai2023functional}, but this time for a PDE-CNN instead of the \(M=\bbM_2\) PDE-G-CNN considered there.

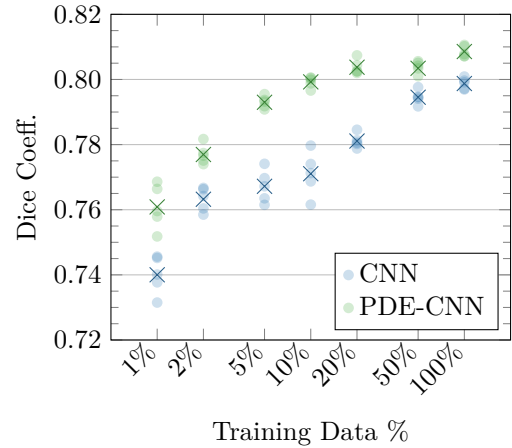
\begin{figure}
\centering
\begin{tikzpicture}
    \begin{axis}[
        width = 0.9\linewidth,
        xticklabel style = {rotate = 45, anchor = east},
        xmode = log,
        xmin  = 0.005,
        xmax  = 2,
        xtick = {0.01, 0.02, 0.05, 0.10, 0.20, 0.50, 1},
        xticklabels = {1\%,2\%,5\%,10\%,20\%,50\%,100\%},
        xlabel = {Training Data \%},
        ymin = 0.72,
        ymax = 0.82,
        ytick       = {0.72,0.74,0.76,0.78,0.80,0.82},
        yticklabels = {0.72,0.74,0.76,0.78,0.80,0.82},
        ylabel = {Dice Coeff.},
        minor y tick num=3, 
        ymajorgrids = true,
        legend pos = south east,
        legend cell align={left},
    ]

        \addplot[
            only marks, 
            fill opacity = 0.25, 
            draw opacity = 0,
            fill = Set1-B
        ] coordinates { 
            (0.01, 0.7456)
            (0.01, 0.7452)
            (0.01, 0.7401)
            (0.01, 0.7377)
            (0.01, 0.7315)
            
            (0.02, 0.7667)
            (0.02, 0.7663)
            (0.02, 0.7642)
            (0.02, 0.7604)
            (0.02, 0.7585)

            (0.05, 0.7741)
            (0.05, 0.7697)
            (0.05, 0.7672)
            (0.05, 0.7615)
            (0.05, 0.7635)
        
            (0.10, 0.77125219)
            (0.10, 0.76157054)
            (0.10, 0.77969627)
            (0.10, 0.76869858)
            (0.10, 0.77404570)

            (0.20, 0.77879498)
            (0.20, 0.78058162)
            (0.20, 0.78021627)
            (0.20, 0.78132745)
            (0.20, 0.78456308)
            
            
            
            (0.50, 0.79764211)
            (0.50, 0.79174041)
            (0.50, 0.79461733)
            (0.50, 0.79463660)
            (0.50, 0.79409555)


            
            

            (1.00, 0.79911892)
            (1.00, 0.79964832)
            (1.00, 0.79715814)
            (1.00, 0.80090074)
            (1.00, 0.79692841)
        };
        \addlegendentry{CNN}
        
        \addplot[
            only marks, 
            fill opacity = 0.25, 
            draw opacity = 0,
            fill = Set1-C
        ] coordinates {
            (0.01, 0.7686)
            (0.01, 0.7664)
            (0.01, 0.7596)
            (0.01, 0.7579)
            (0.01, 0.7518)

            (0.02, 0.7817)
            (0.02, 0.7774)
            (0.02, 0.7766)
            (0.02, 0.7740)
            (0.02, 0.7751)
            
            (0.05, 0.7955)
            (0.05, 0.7937)
            (0.05, 0.7932)
            (0.05, 0.7918)
            (0.05, 0.7908)
            
            (0.10, 0.8006)
            (0.10, 0.8004)
            (0.10, 0.7999)
            (0.10, 0.7988)
            (0.10, 0.7966)
            
            (0.20, 0.8074)
            (0.20, 0.8044)
            (0.20, 0.8025)
            (0.20, 0.8025)
            (0.20, 0.8020)
            
            
            
            (0.50, 0.8056)
            (0.50, 0.8051)
            (0.50, 0.8042)
            (0.50, 0.8034)
            (0.50, 0.8010)
            
            
            
            
            
            (1.00, 0.8106)
            (1.00, 0.8102)
            (1.00, 0.8078)
            (1.00, 0.8072)
            (1.00, 0.8071)
        };
        \addlegendentry{PDE-CNN}
        
        \addplot[
            only marks, 
            color = Set1-B!75!black,
            mark = x, 
            dashed,
            mark options={solid},
            mark size=4
        ] coordinates { 
            (0.01, 0.74002)
            (0.02, 0.76322)
            (0.05, 0.76720)
            (0.10, 0.77105265)
            (0.20, 0.78109668)
            (0.50, 0.79454640)
            (1.00, 0.79875091)
        };

        \addplot[
            only marks, 
            color = Set1-C!75!black,
            mark = x, 
            dashed,
            mark options={solid},
            mark size=4
        ] coordinates {
            (0.01, 0.76086)
            (0.02, 0.77696)
            (0.05, 0.79300)
            (0.10, 0.79926)
            (0.20, 0.80376)
            (0.50, 0.80343)
            (1.00, 0.80858)
        };
    \end{axis}
\end{tikzpicture}
\caption{A scatterplot of the accuracy of a 6-layer 24-channel CNN (\(31\, 488\) parameters) and PDE-CNN (\(5\,280\) parameters) on the DRIVE dataset, when trained multiple times on varying amounts of training data. The crosses indicate the mean.}
\label{fig:data_efficiency}
\end{figure}

\section{Conclusion} \label{sec:conclusion}

PDE-CNNs are an interesting alternative to CNNs in the sense that their constituents, this being solvers of PDEs that generate scale-spaces, are geometrically meaningful and interpretable.

The existing PDE-CNN framework uses four PDEs: convection, diffusion, dilation, and erosion.
Through the introduction of semifield scale-spaces (\Cref{def:semifield_scale_space}) we demonstrate the presence of a broad class of PDEs that remain unused within the PDE-CNN paradigm.

The theory of semifields scale-spaces is expressive and encapsulates a large class of known scale-spaces.
\Cref{res:explicit_form_reduced_kernel} shows that every semifield gives rise to a one-parameter family of semifield scale-spaces.
This indicates that the generalization to semifields is one that is not too general and definitely fruitful.

In \Cref{sec:data_efficiency} we empirically verified that on the DRIVE dataset that PDE-CNNs, just like PDE-G-CNNs, when compared to traditional CNNs, require less training data, have fewer parameters, and increased accuracy.

In \Cref{sec:experiment_semifields} we experimented on the inclusion of various semifields and their corresponding scale-spaces within PDE layers of a PDE-CNN.
We see that the thought ``more semifields means better accuracy'' is incorrect, and that it is not clear if the addition of more semifields into the already existing PDE-CNN framework is worth the effort. 
However, in all cases inclusion of the tropical semifield improved the result, advocating for tropical algebras in PDE-based neural networks.

\subsection*{Further Research}

When comparing the results of PDE-CNNs on the DRIVE dataset here to the \(G=\text{SE}(2)\) PDE-G-CNNs results in \cite{pai2023functional}, the accuracy is essentially the same (Dice \(\approx 0.81\)), but there is a trade-off between memory usage and parameter reduction. 

The \(\text{SE}(2)\) variant has less parameters (\(2\,560\)) but, due to the feature maps being scalar fields on \(\text{SE}(2)\), uses more memory, that being O\texttimes H\texttimes W scalars per feature map.
Here O refers to the amount of orientations (typically 8), H to the height of the images, and W to the width.

Conversely, the \(\bbR^2\) variant has more parameters (\(5\,280\)) but uses much less memory; H\texttimes W scalars per feature map.
This means that in some applications the PDE-CNN architecture might be preferable.
However, the goal of the work here was not to compare PDE-CNNs to PDE-G-CNNs and the observations here only apply to the DRIVE dataset.
Further research is needed to properly compare both architectures.

\backmatter

\section*{Declarations}

\subsection*{Acknowledgements}

We thank Adrien Castella \cite{castella2021introduction} for the initial Python implementation of the \(\bbR^2\) convection, diffusion, dilation and erosion PDE sublayers 
\url{https://github.com/adrien-castella/PDE-based-CNNs}.

\if\arxiv0
    \subsection*{Author Contributions}
    
    G. Bellaard is the first author and main writer of the manuscript.
    G. Bellaard, S. Sakata, B. Smets, and R. Duits contributed to the semifield theory.
    G. Bellaard, S. Sakata, and R. Duits contributed to the semifield scale-space theory.
    S. Sakata performed the initial exploratory experiments.
    G. Bellaard prepared, conducted, and analyzed the final experiments.
    B. Smets provided important practical advice on the experimental aspects.
    All authors collaborated closely and reviewed the manuscript carefully.
    
    Author contributions per section. 
    \Cref{sec:introduction}: G. Bellaard, R. Duits; 
    \Cref{sec:background}: G. Bellaard, R. Duits; 
    \Cref{sec:semifield_theory}: G. Bellaard, S. Sakata, B. Smets, R. Duits; 
    \Cref{sec:semifield_scale_space}: G. Bellaard, S. Sakata, R. Duits; 
    \Cref{sec:consequences}: G. Bellaard, S. Sakata, R. Duits;
    \Cref{sec:architecture}: G. Bellaard, B. Smets, R. Duits;
    \Cref{sec:experiments}: G. Bellaard, S. Sakata, B. Smets;
    \Cref{sec:conclusion}: G. Bellaard, R. Duits.
\fi

\subsection*{Funding}

The EU is gratefully acknowledged for financial support through the REMODEL project (MSCA-SE 101131557). 

We gratefully acknowledge the Dutch Foundation of Science NWO for funding of VICI 2020 Exact Sciences (Duits, Geometric learning for Image Analysis, VI.C.202-031 \url{https://www.nwo.nl/en/projects/vic202031}).

\subsection*{Availability of Data and Code}

{\if\arxiv0\color{RoyalBlue}\fi
All code be found at \url{https://gitlab.com/gijsbel/semifield-pde-cnns}.
}
The DRIVE dataset \cite{staal2004ridge} can (currently\footnote{The old address was \url{https://web.archive.org/web/20191003101812/http://www.isi.uu.nl/Research/Databases/DRIVE/}.}) be found at \url{https://drive.grand-challenge.org/}.
The LieTorch package is public and can be found at \url{https://gitlab.com/bsmetsjr/lietorch}.
The original PDE-CNN implementation can be found at \url{https://github.com/adrien-castella/PDE-based-CNNs}. 

\if\arxiv0
    \subsection*{Competing Interests}
    
    R. Duits is a member of the editorial board of Journal of Mathematical Imaging and Vision (JMIV).
\fi

\begin{appendices}

\section{Semifield Fourier Transforms} \label{sec:employed_fourier_satisfy_definition}

\begin{lemma} \label{res:employed_fourier_satisfy_definition}
    The employed semifield Fourier transforms satisfy \Cref{def:semifield_fourier_transform}.
\end{lemma}

\begin{proof}
    In the linear semifield \(L\) case we know that the familiar Fourier transform satisfies the definition.
    
    As for the root and logarithmic semifields, being isomorphic to the linear semifield, we can quickly deduce that they also satisfy the definitions through the equalities
    \begin{equation}
    \begin{split}
        \cF_{L_\mu} &= \cP_{\varphi_\mu^{-1} } \circ \cF_{L} \circ \cP_{\varphi_\mu},\\
        \cF_{R_p} &= \cP_{\varphi_p^{-1} } \circ \cF_{L} \circ \cP_{\varphi_p},
    \end{split}
    \end{equation}
    where \(\cP\) is the pointwise operator \eqref{eq:pointwise_operator}, \(\varphi_\mu(x) = e^{\mu x}\) the semifield isomorphism \(\varphi_\mu : L_\mu \to L_{\geq 0}\), and \(\varphi_p(x) = x^p\) the semifield isomorphism \(\varphi_p : R_p \to L_{\geq 0}\).
    For example, to show that \(\cF_{L_\mu}\) satisfies the convolution property:
    \begin{equation}
    \begin{split}
        &\cF_{L_\mu}(f \rconv g) \\
        &= \cP_{\varphi_\mu^{-1} } \cF_{L} \cP_{\varphi_\mu}(f \rconv g) \\
        &= \cP_{\varphi_\mu^{-1} } \cF_{L}((\cP_{\varphi_\mu} f) * (\cP_{\varphi_\mu} g))\\
        &=\cP_{\varphi_\mu^{-1} } ( ( \cF_{L} \cP_{\varphi_\mu} f ) \times ( \cF_{L} \cP_{\varphi_\mu} g ) )\\
        &=  ( \cP_{\varphi_\mu^{-1} }  \cF_{L} \cP_{\varphi_\mu} f ) \rt ( \cP_{\varphi_\mu^{-1} } \cF_{L} \cP_{\varphi_\mu} g )\\
        &= (\cF_{L_\mu} f) \rt (\cF_{L_\mu} g),
    \end{split}
    \end{equation}
    where \(\rconv\) and \(\rt\) are the semifield convolution and multiplication of \(L_\mu\) and where $\times$ denotes the standard pointwise product of functions. In the above derivation we have used that
    \begin{equation}
    \begin{split}
        \cP_{\varphi_\mu}(f \rconv g) &= (\cP_{\varphi_\mu} f) * (\cP_{\varphi_\mu} g),\\
        \cP_{\varphi_\mu^{-1}}(f \times g) &= (\cP_{\varphi_\mu^{-1}} f) \rt (\cP_{\varphi_\mu^{-1}} g),
    \end{split}
    \end{equation}
    and that \(\cF_{L}\) has the convolution property.

    Consider now the tropical max semifield \(\tropmax\).
    That \(\cF_\tropmax\) satisfies the linearity, equivariances, and the zero-frequency properties is immediate.
    As for the convolution property we have
    \begin{equation}
    \begin{split}
        &(\cF_{\tropmax}(f \rconv g))(\omega)\\
        &= \sup_{x}\ (f \rconv g)(x) - \omega \cdot x \\
        &= \sup_{x}\ (\sup_{y} f(x-y) + g(y)) - \omega \cdot x \\
        &= \sup_{x}\ (\sup_{x_1 + x_2 = x} f(x_1) + g(x_2)) - \omega \cdot x \\
        &= \sup_{x_1,x_2}\ f(x_1) + g(x_2) - \omega \cdot (x_1 + x_2) \\
        &= (\sup_{x_1} f(x_1) - \omega \cdot x_1) + (\sup_{x_2} g(x_2) - \omega \cdot x_2)\\
        &= (\cF_{\tropmax} f)(\omega) \rt (\cF_{\tropmax} g)(\omega),
    \end{split}
    \end{equation}
    where \(\rt\) and \(\rconv\) are the tropical max \(\tropmax\) multiplication and convolution.
    For the invertibility we refer to the Fenchel biconjugation theorem \cite[Thm.4.2.1]{borwein2006convex}.
    That the tropical min semifield Fourier transform \(\cF_\tropmin\) satisfies all properties follows from the fact that \(\tropmin\) is semifield isomorphic to \(\tropmax\) with the isomorphism being \(\phi(x) = -x\).
\end{proof}



\section{Tropical Integration}\label{sec:tropical_integration_correct}
\begin{proposition}
    The natural integration of sum-approachable and bounded from above functions \(f : \bbR^2 \to \tropmax\) is
    \begin{equation}
        \oint^{\tropmax} f = \sup_{x \in \bbR^2} f(x).
    \end{equation}
    The natural integration of sum-approachable and bounded from below functions \(f : \bbR^2 \to \tropmin\) is
    \begin{equation}
        \oint^{\tropmin} f = \inf_{x \in \bbR^2} f(x).\\
    \end{equation}
\end{proposition}
\begin{proof}
    We will only prove this for the tropical max semifield \(\tropmax\), the tropical min semifield case goes completely analogously.
    As we are working with \(\tropmax\) we remind ourselves that we have \(\rp=\max\), \(\rt = +\), \(\mu_{\tropmax}(A)=0\), and \[\ro_A(x) = \begin{cases}
            0 & \text{ if } x \in A\\
            -\infty & \text{ otherwise }
        \end{cases}. \]
    The function $f$, being sum-approachable and bounded from above, is pointwise defined by the limit
    \begin{equation}
        f(x) = \lim_{n \to \infty} \bigoplus_{i=1}^n a_i \rt \ro_{A_i}(x),
    \end{equation}
    with \(A_i\) non-empty and \(a_i\) bounded from above.
    We define its natural integral by
    \begin{equation}
    \begin{split}
        \oint f 
        &:= \lim_{n \to \infty} \oint \Par{\bigoplus_{i=1}^n a_i \rt \ro_{A_i}}\\
        &= \lim_{n \to \infty} \bigoplus_{i=1}^n a_i \rt \mu(A)\\
        &= \lim_{n \to \infty} \max_{i=1,\dots,n} a_i + 0\\
        &= \sup_{i \in \bbN} a_i,\\
    \end{split}
    \end{equation}
    where the second equality is by the linearity of the integration \eqref{eq:linearity_integration} and the indicator function property \eqref{eq:semifield_integration_indicator_function}.
    
    Similarly, we have
    \begin{equation}
    \begin{split}
        \sup_{x \in \bbR^2} f(x) 
        &= \sup_{x \in \bbR^2} \lim_{n \to \infty} \bigoplus_{i=1}^n a_i \rt \ro_{A_i}(x)\\
        &= \sup_{x \in \bbR^2} \lim_{n \to \infty} \max_{i=1,\dots,n} a_i + \ro_{A_i}(x)\\
        &= \sup_{x \in \bbR^2} \sup_{i \in \bbN} a_i + \ro_{A_i}(x) \\
        &= \sup_{i \in \bbN} \sup_{x \in \bbR^2} a_i + \ro_{A_i}(x) \\
        &= \sup_{i \in \bbN} a_i,
    \end{split}
    \end{equation}
    where in the fourth equality we interchanged the order of suprema, and in the fifth equality we used the definition of \(\ro_{A_i}\).

    Combining these two results, it follows that
    \begin{equation}
        \oint f = \sup_{x \in \bbR^2} f(x)
    \end{equation}
    is the natural tropical max semifield integration.
\end{proof}

\end{appendices}

\bibliography{references.bib}

\end{document}